\documentclass[conference]{IEEEtran}
\IEEEoverridecommandlockouts
\usepackage{cite}
\usepackage{amsmath,amssymb,amsfonts}
\usepackage{graphicx}
\usepackage{textcomp}
\usepackage{xcolor}

\usepackage[ruled,linesnumbered,vlined]{algorithm2e}
\usepackage{amsmath}
\usepackage{mathtools}
\usepackage{amsthm}
\usepackage{hyperref}
\usepackage[capitalize,nameinlink,noabbrev]{cleveref}
\usepackage{booktabs}
\usepackage{makecell}
\usepackage{siunitx}
\usepackage{bbold}
\usepackage{multirow}
\usepackage{physics}
\usepackage[caption=false,font=footnotesize]{subfig}
\usepackage{pifont}
\usepackage{tabularx}
\usepackage{xfrac}
\usepackage{hyphenat}

\newcommand{\tool}{\textsc{DP-Hype}}
\newcommand{\numclients}{n}
\newcommand{\dataset}{D}
\newcommand{\hpset}{H}
\newcommand{\numhps}{p}
\newcommand{\priveps}{\varepsilon}
\newcommand{\privdelta}{\delta}
\newcommand{\numlocaltops}{k}
\newcommand{\noisemultiplier}{\sigma}
\newcommand{\privrdpeps}{\varepsilon'} 
\newcommand{\algo}{\mathcal{M}} 
\newcommand{\votingvec}{\vb{v}}

\newcommand{\ltwosens}{\Delta_2}
\newcommand{\loss}{\mathcal{L}}
\newcommand{\opt}{\textsc{Opt}}
\newcommand{\randguess}{\textsc{RandGuess}}
\newcommand{\dirichletparam}{\beta}

\DeclarePairedDelimiterX{\infdivx}[2]{(}{)}{%
  #1\;\delimsize\|\;#2%
}
\newcommand{\infdiv}{\mathcal{D}_\alpha\infdivx}

\newtheorem{theorem}{Theorem}[section]
\newtheorem{definition}{Definition}[section]
\newtheorem{lemma}{Lemma}[section]

\DeclareMathOperator*{\argmax}{argmax}
\DeclareMathOperator*{\argmin}{argmin}

\let\oldnl\nl
\newcommand{\nonl}{\renewcommand{\nl}{\let\nl\oldnl}}

\newcommand{\xmark}{\textcolor{red}{\ding{55}}}
\newcommand{\bcmark}{\textcolor{green!60!black}{\ding{52}}}

\newcommand{\commbx}[2]{%
  \makebox[2.7em][r]{$#1$} /%
  \makebox[2.7em][r]{$#2$}%
}

\begin{document}

\title{\tool: Federated Differentially Private Hyperparameter Search}

\author{
\IEEEauthorblockN{
Johannes Liebenow \hspace{6mm}
Thorsten Peinemann \hspace{6mm}
Esfandiar Mohammadi
}
\vspace{1.5mm}
\IEEEauthorblockN{
Institute of IT Security, University of Luebeck, Luebeck, Germany
}
\vspace{1.5mm}
\IEEEauthorblockN{
\{j.liebenow, t.peinemann, esfandiar.mohammadi\}@uni-luebeck.de
}
}

\maketitle

\begin{abstract}
Tuning hyperparameters in federated machine learning can substantially impact model performance. When hyperparameters are tuned on sensitive data, privacy becomes an important challenge and to this end, differential privacy has emerged as the de facto standard for provable privacy. A standard setting in federated learning is that clients agree on a shared setup, i.e., find a compromise from a set of hyperparameters, like a model's learning rate.
Yet, prior work on privacy-preserving hyperparameter tuning is tailored to specific learning tasks, does not account for the privacy leakage of aggregated results, or offers a sub-optimal privacy-utility trade-off.

In this work, we present our algorithm \tool{}, which performs a federated and privacy-preserving hyperparameter search by conducting a federated voting based on local hyperparameter evaluations of clients. In this way, \tool{} selects hyperparameters that lead to a compromise supported by a majority of clients, while maintaining scalability and independence from specific learning tasks. We prove that \tool{} preserves the strong notion of differential privacy called client-level differential privacy and, importantly, show that its privacy guarantees do not depend on the number of hyperparameters. We also provide bounds on its utility guarantees, that is, the probability of finding good hyperparameters, and implement \tool{} as a submodule in the popular Flower framework for federated machine learning. In addition, we evaluate performance on multiple benchmark data sets in iid as well as multiple non-iid settings and demonstrate high utility of \tool{} even under small privacy budgets. 
\end{abstract}


Federated learning has made significant advances, yet many learning methods have to be fine-tuned to acceptable hyperparameters before achieving strong utility. 
However, even the process of finding suitable hyperparameters introduces a privacy risk. Outliers can noticeably influence a hyperparameter search, and an adversary can use this information to infer whether certain data points were part of the process or not~\cite{papernot2021hyperparameter}. Consequently, privacy needs to be preserved. For provable privacy guarantees, differential privacy has become the de facto standard. 
Conducting a federated and scalable hyperparameter search while preserving differential privacy is a challenging task: The key challenge is that usually a lot of hyperparameters have to be evaluated, but revealing the results of the evaluations to other parties can leak information about training data, thereby deteriorating the differential privacy (DP) guarantees. 
While in the centralized setting, several techniques~\cite{liu2019private,papernot2021hyperparameter,koskela2023practical} have been proposed that are independent of the number of hyperparameters, these techniques rely on one crucial assumption: The adversary is not able to see intermediate evaluation results, but only sees the final hyperparameter that was chosen. This assumption is not satisfied in the federated setting, and it is not clear how to do so in a scalable and federated manner, i.e., without relying on computation parties or trusted third parties, and without paying the well-known $\sqrt{\#\mathrm{clients}}$-price of local DP techniques~\cite{balle2019blanket}.

\begin{table}[t]
\centering
\caption{Overview of related work for privacy-preserving federated hyperparameter search in comparison with \tool{}.
\textbf{Task Agnostic:} The hyperparameter search can be applied to any learning task with some sort of loss function. 
\textbf{DP Notion:} If the algorithm satisfies differential privacy (DP), which notion does it satisfy exactly. 
\textbf{DP Accounting indep. of \#HPs:} The accounting for the privacy budget does not depend on the (evaluated) number of hyperparameters. Parentheses indicate that the specific property is not fully fulfilled.
}
\begin{tabular}{@{}lccc@{}}
\toprule
\textbf{Algorithm}
& \makecell{\textbf{Task}\\\textbf{Agnostic}}
& \makecell{\textbf{DP}\\\textbf{Notion}}
& \makecell{\textbf{DP Accounting}\\\textbf{indep. of \#HPs}} \\
\midrule
PrivTuna~\cite{mitic2025privacy}       & \bcmark    & \xmark       & -          \\
LDP-DSBO~\cite{chen2024locally}        & \xmark     & local        & (\bcmark)  \\
DP-FTS-DE~\cite{dai2021differentially} & (\bcmark)  & central      & (\bcmark)  \\
Feathers~\cite{seng2022feathers}       & \bcmark    & (local)      & \xmark     \\
\textbf{\tool{} (Ours)}                & \bcmark    & client-level & \bcmark    \\
\bottomrule
\end{tabular}
\label{tbl::overviewRelated}
\end{table}

In the federated setting, there are essentially two scenarios concerning the data distribution: either the local data distributions are sufficiently similar such that looking for a shared set of hyperparameters makes sense, or the local data distributions are vastly different such that each party (or group of parties) have to search for their own hyperparameter set. In this work, we focus on the first scenario where we aim to find a compromise among clients in terms of hyperparameters, e.g. the model learning rate or its momentum. Importantly, finding a compromise does not imply that clients follow the same data distribution (iid). In reality, their data often differs significantly (non-iid), which means that a federated hyperparameter search algorithm must be able to cope with these variations in local distributions to a reasonable extent.   

In this context, there exist some algorithms from prior work that aim for a similar goal (see \Cref{tbl::overviewRelated}).
The first solution is PrivTuna~\cite{mitic2025privacy}, which hides intermediate results by using secure multiparty computation (SMPC). However, PrivTuna does not take into account privacy leakage of the final output.
Furthermore, LDP-DSBO~\cite{chen2024locally} is designed only to optimize certain continuous hyperparameters in learning setups that require federated model training via gradient descent, and this approach is tailored for local DP, which typically requires adding substantial noise for reasonable privacy budgets.
DP-FTS-DE~\cite{dai2021differentially} must rely on a trusted third party to achieve central DP. In addition, each client converges to different hyperparameters, which is why this algorithm is not fully task agnostic as an additional phase would be needed to agree on a shared setup.
Feathers~\cite{seng2022feathers} aims to find a hyperparameter based on loss differences compared to other candidates. However, a closer look at the algorithm reveals that Feathers uses an unbounded loss function without any clipping and its privacy accounting does not include the number of iterations. Thus, local DP cannot be preserved. 
All prior differentially private algorithms have in common that their privacy accounting, the process of tracking where and how much of the privacy budget is spent, is somewhat dependent on the hyperparameters. For both DP-FTS-DE and LDP-DSBO, properties of the hyperparameter search space affect privacy indirectly by influencing how many iterations are needed to reach a target utility. In LDP-DSBO the dimensionality of the released hyperparameter can also affect privacy directly through sensitivity. However, Feathers has the strongest dependence as the number of evaluated candidates immediately requires composition.  

To close this gap, we introduce \tool{}, a new algorithm for a privacy-preserving federated hyperparameter search. 
Our approach is based on the observation that if the data distribution per client does not vary too much then local evaluations correlate with the performance of hyperparameters evaluated on the entire data set~\cite{mitic2025privacy}. 
Thus, \tool{} operates in a cross-silo setting where clients have enough computation power to train machine learning models locally.
On the technical side, we leverage the voting approach of \cite{feddpvoting} in combination with a special privacy accounting~\cite{kairouz2021distributed}. Clients evaluate each proposed hyperparameter locally and are allowed to vote for a few of them. Votes are aggregated using a secure summation protocol, which enables \tool{} to satisfy client-level DP.
This counterintuitive trade, i.e., exchanging the expressiveness of global evaluations with local ones, results in a massive privacy amplification. Local training does not have to preserve DP, and we can reduce the required privacy budget to a minimum while still preserving utility. In particular, the privacy guarantees of \tool{} are independent of the total number of hyperparameters, which is important because the set of hyperparameters is usually large and having to account for each candidate would deteriorate utility quickly. Satisfying client-level DP also offers \tool{} an improved privacy-utility trade-off as an adversary only sees aggregated results, which effectively improves the signal-to-noise ratio for a growing number of clients. In addition, communication overhead is minimal, as each client's output only scales linearly in the number of hyperparameters, and our voting approach can be realized with a single invocation of secure summation.

In summary, we make the following contributions:
\begin{itemize}
    \item We present \tool{}, the first algorithm for a federated hyperparameter search that is task agnostic, preserves client-level DP, and whose privacy guarantees are independent of the number of hyperparameters. In addition, the simple yet efficient design, including a secure summation protocol, makes \tool{} scale to many clients.
    \item We show that \tool{} satisfies client-level DP and quantify the probability of achieving a good compromise. In addition, we provide insights on the behavior of \tool{} by simulating various scenarios.  
    \item \tool{} is implemented as a practical and ready-to-use submodule in Flower, a state-of-the-art framework for federated machine learning.\footnote{
    The source code of our implementation is publicly available at \href{https://github.com/UzL-PrivSec/dp-hype}{\url{https://github.com/UzL-PrivSec/dp-hype}}.}
    \item We perform an evaluation on three benchmark data sets and show that \tool{} performs well in iid as well as non-iid scenarios, even under small privacy budgets.
\end{itemize}

The remainder of this work is structured as follows: 
In~\Cref{sec::prelims}, we introduce differential privacy in combination with secure summation, our threat model, and some notation. In \Cref{sec::dphype}, we present the setting of differentially private federated hyperparameter search, the intuition behind our approach, and the technical details of our algorithm \tool{}. This is followed by a privacy proof (\Cref{sec::privacyproof}) and a utility analysis (\Cref{sec::utility}) in combination with an ablation study based on simulated clients. Then, we present our empirical evaluation of \tool{} on benchmark data sets~(\Cref{sec::eval}). 
Afterward, we discuss important design choices of \tool{}~(\cref{sec:discussion}).
We conclude with an overview of the related work (\Cref{sec::relatedwork}) and our conclusion (\Cref{sec::conclusion}).

\section{Preliminaries}
\label{sec::prelims}
In this section, we introduce the key concepts used in this work, namely differential privacy as well as secure summation, and present our threat model including some notation.

\subsection{Differential Privacy}
Differential privacy (DP) was first introduced by~\cite{dwork2006differential} and provides a mathematical framework to protect personal information of individuals when applying an algorithm to sensitive data. 
If an algorithm preserves differential privacy, then each client can plausibly deny that their data was part of the input when the attacker observes a specific output.
In a setting where the data set is federated among clients and the server is trusted, there are essentially two scenarios called client-level and record-level DP. Client-level DP aims to hide the influence of an entire client's data, in contrast to record-level DP, where an adversary knows all data points of all clients except for one data point of the targeted client. In this work, we consider client-level DP to achieve the strongest privacy protection.

In mathematical terms, it means that the ratio of output distributions for specific inputs when replacing data from a client is bounded by $\priveps$ except for a probability mass $\privdelta$. In this context, $\priveps$ is called the privacy budget and regulates the degree of privacy protection. DP naturally introduces a privacy-utility trade-off as more privacy budget means that the output of the corresponding algorithm can depend more on the input to achieve a higher degree of utility but also less privacy protection, and vice versa.

Let $\mathcal{\dataset}$ be the set of all data sets. Given data sets $\dataset = \dataset_1 \cup \dots \cup \dataset_\numclients \in \mathcal{\dataset}$ and $\dataset' = (\dataset \setminus \dataset_i)  \cup \dataset'_i$ for some client $i$, the data sets $\dataset$ and $\dataset'$ are called \emph{neighboring}.
\begin{definition}[Differential Privacy~\cite{dwork2006differential}]
    Given $\varepsilon \ge 0$ and $\delta \in [0,1]$, then a randomized algorithm $\mathcal{M}: \mathcal{\dataset} \rightarrow \mathcal{Y}$ preserves differential privacy if for all neighboring data sets $\dataset, \dataset'  \in \mathcal{\dataset}$ and all measurable sets $S \subseteq \mathcal{Y}$:
    \[
    \Pr[\mathcal{M}(\dataset) \in S] \leq e^\varepsilon \Pr[\mathcal{M}(\dataset') \in S] + \delta.
    \]
\end{definition}
Two useful properties of DP are post-processing and sequential composition. DP is closed under any post-processing as long as it is not based on sensitive data. The composition of multiple DP algorithms also yields a DP algorithm where the exact guarantees depend on the setting, either adaptive or non-adaptive, and on the composition theorem itself~\cite{dwork2006differential}. In the worst case, however, $\varepsilon$ scales roughly by the square root of the number of compositions while $\privdelta$ composes additively. 

We use an alternative flavor of differential privacy, called Rényi-DP (RDP), first introduced by~\cite{mironov2017renyi}. RDP is based on the Rényi divergence of the output distribution of an algorithm regarding neighboring data sets, and naturally translates to DP. 
\begin{definition}[Rényi Differential Privacy~\cite{mironov2017renyi}]
    A randomized algorithm $\mathcal{M}$ preserves $\varepsilon$-RDP of order $\alpha$ if for all neighboring data sets $\dataset, \dataset' \in \mathcal{\dataset}$:
    \[
        \infdiv{\mathcal{M}(\dataset)}{\mathcal{M}(\dataset')} \leq \varepsilon
    \]
    where $\infdiv{.}{.}$ is the Rényi divergence. 
\end{definition}
RDP can be transformed back to approximate DP for any fixed $\alpha$. To the best of our knowledge, the currently tightest conversion is the following:
\begin{theorem}[From RDP to DP~\cite{balle2020hypothesis}]
\label{th::rdpToAdp}
    If a randomized algorithm $\mathcal{M}$ satisfies $(\alpha, \varepsilon)$-RDP then it also satisfies $(\varepsilon + \log((\alpha-1)/\alpha) - (\log \delta + \log \alpha)/(\alpha-1), \delta)$-DP for any $\delta \in (0,1)$.
\end{theorem}
On this basis, we introduce the Gaussian mechanism, which is an important building block in this work. The mechanism adds Gaussian noise to the output of a function $f$ calibrated on the L$2$ sensitivity $\Delta_2f$ of that function and privacy parameters $\priveps,\privdelta$. In this context, $\Delta_2f = \max_{\dataset, \dataset'} ||f(\dataset) - f(\dataset')||_2$ is the worst-case change in the L$2$ norm in the output that can be observed by the adversary regarding all possible pairs of neighboring data sets.
\begin{definition}[Gaussian Mechanism \cite{mironov2017renyi}]
\label{def:multvarigaussian}
    Let $\algo: \dataset \rightarrow \mathbb{R}^d$ be defined as $\algo(\dataset) = f(\dataset) + \vb{z}$ for some function $f: \dataset \rightarrow \mathbb{R}^d$ with $\vb{z} \sim \mathcal{N}(0, \sigma^2\mathbf{I}_d)$. Then $\algo$ satisfies $(\alpha,\frac{\alpha(\Delta_2f)^2}{2\sigma^2})$-RDP.
\end{definition}

\paragraph{Simulated Central Scenario using SMPC}
In the federated learning scenario, without additional tools, clients have to preserve privacy on their own, which is called local DP (LDP). However, LDP usually requires a lot of noise and can drastically reduce utility~\cite{balle2019blanket}. Clients could also entrust the server with their sensitive data, yet this is an unrealistic trust assumption. Consequently, secure multi-party computation (SMPC) is often used. With SMPC, clients can compute a function securely without having to reveal their inputs to the adversary, except for the final result of the computation. Hiding intermediate results can be a massive privacy amplification, as the central scenario can be simulated. In terms of DP, if the adversary can only learn the final result of the computation and cannot single out sensitive information of individual clients, then there is no difference to the central scenario, where the adversary by definition can only observe the final output of the algorithm. However, most SMPC solutions come at the cost of being scalable to no more than a handful of clients~\cite{wirth2022easysmpc,gamiz2025challenges,zhou2024secure,hu2021make}. But in the special case of computing a federated sum, secure summation protocols have proven highly scalable with low resource overhead~\cite{secsum_bell}. Throughout this work, we use the terms secure aggregation and secure summation interchangeably.

If each of the $n$ clients locally adds noise drawn from $\mathcal{N}(0,\sigma^2/\numclients)$, then by applying secure summation, the guarantees of the Gaussian mechanism hold although the noise is added partially by each client. This directly follows from the fact that the sum of $\numclients$ independent normal random variables with variance $\sigma^2 / \numclients$ is normal with variance $\numclients \cdot (\sigma^2 / \numclients) = \sigma^2$.

\subsection{Threat \& System Model}
We assume a data set $\dataset = \dataset_1 \cup \dots \cup \dataset_\numclients \in \mathcal{D}$ that is federated among $\numclients$ clients, where each client $i$ holds their own private subset $\dataset_i$. Important for DP, each user only contributes data to a single client's data set. For ease of presentation, we assume the availability of a server that coordinates the protocol. However, such a server is not mandatory for our algorithm and the same role could also be fulfilled by one or more clients.  
When results derived from sensitive data need to be shared, clients interact with the server only through a secure summation protocol.

We are in the cross-silo setting, where clients are assumed to be capable of running part of the learning algorithm locally, such as training a model on their private data. This is a standard setting, especially when, for example, gradients are shared with a server, as this already requires local training~\cite{fedavg}.
The server does not require special computational resources since most of the overhead is caused by the secure summation protocol, which is considered to entail only a small overhead. 

We assume an honest-but-curious model for both the server and clients, which is in line with recent work~\cite{geyer2017differentially,mcmahan2017learning}. 
Our algorithm is explicitly designed to preserve differential privacy without relying on a trusted third party and the adversary is assumed to reside on the server side.

\subsection{Notation}
In this work, we adopt the following notation. The normal distribution with mean $\mu$ and standard deviation $\sigma$ is denoted by $\mathcal{N}(\mu,\sigma^2)$. The normal CDF with $\mu=0$ is denoted as $\Phi_{\sigma^2}$. Vectors are represented using small bold letters. A random vector $\vb{z} \sim \mathcal{N}(0,\sigma^2\mathbf{I}_d)$ indicates that $\vb{z}$ is sampled from a $d$-dimensional multivariate normal distribution, where $\mathbf{I}_d \in \mathbb{R}^{d \times d}$ denotes the identity matrix of dimension $d$. The $d$th component of a vector $\vb{v}$ is accessed as $\vb{v}[d]$.  Furthermore, we say that $[t]$ denotes the index set $\{1,\dots,t\} \subseteq \mathbb{N}$.

\section{DP-Hype}
\label{sec::dphype}
In this section, we introduce \tool{}, an algorithm to perform a federated privacy-preserving hyperparameter search. First, we discuss the setting of such a hyperparameter search. Then, we explore the high-level idea behind \tool{} and present technical details as well as the pseudocode.

\subsection{Problem Statement}
\begin{figure*}[t]
    \includegraphics{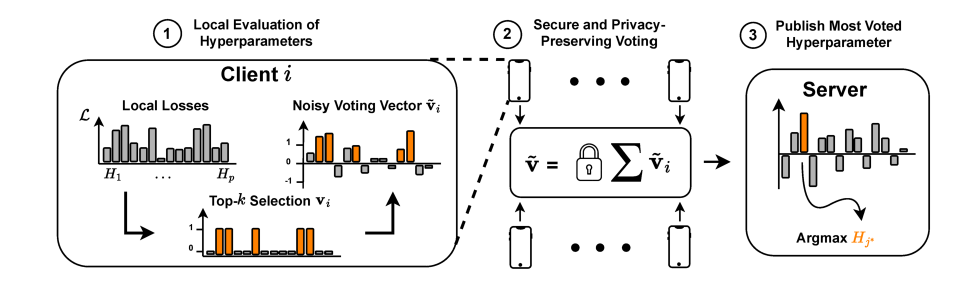}
    \caption{A conceptual overview of our privacy-preserving algorithm for federated hyperparameter search \tool{}. The set of hyperparameters $\hpset = \{\hpset_1, \dots, \hpset_\numhps\}$ and the loss function $\loss$ for the given learning task are publicly available. First, each client locally computes the loss for each hyperparameter and selects the top-$\numlocaltops$ hyperparameters with the smallest loss. A voting vector $\votingvec_i$ is created, with a $1$ on each position corresponding to the top-$\numlocaltops$ hyperparameters and a $0$ elsewhere. Each $\votingvec_i$ is noised and entry-wise aggregated on the server side via secure summation to obtain $\tilde{\votingvec}$. The server then outputs the hyperparameter with the most noisy votes.}
\end{figure*}

Hyperparameter search, also referred to as hyperparameter optimization or tuning, is the task of finding hyperparameters for a specific learning problem that maximize the utility of this task. Formally, let $\hpset \subseteq \mathcal{H}$ denote a set of candidate hyperparameters from the space of possible hyperparameters $\mathcal{H}$. The goal is to identify a hyperparameter $\hpset_{j^*} \in \hpset$ that minimizes a loss function $\loss: \mathcal{H} \times \mathcal{D} \rightarrow \mathbb{R}$ on a given data set $\dataset \in \mathcal{D}$ such that
\begin{equation}
\label{eq:globalargmin}
\hpset_{j^*} = \underset{\hpset_j \in \hpset}{\argmin} \ \loss(\dataset, \hpset_j).
\end{equation}
This is the centralized formulation of hyperparameter search, but it becomes substantially more challenging in federated settings.
In federated learning, hyperparameter search can manifest in two fundamentally different ways. If the data across clients is highly non-iid, then different subpopulations typically emerge, each requiring their own hyperparameters, which may differ significantly from those of other clients. In such cases, an additional phase is needed to identify subpopulations and assign suitable hyperparameters to them. On the other hand, when the clients' data does not differ too much, hyperparameter search reduces to finding common hyperparameters that allow clients to perform a learning task jointly. Many federated algorithms such as FedAvg~\cite{fedavg} or FedAdam~\cite{reddi2020adaptive} rely on clients sharing a similar setup, which essentially requires a compromise over hyperparameters. Such a scenario, where clients must agree on shared hyperparameters in a federated manner, is the focus of this work.

Before introducing our approach, it is important to understand why naively optimizing \Cref{eq:globalargmin} in a federated context is infeasible in realistic scenarios, i.e., when the number of hyperparameters is large. Without structural knowledge about the behavior of $\loss(\dataset,\hpset)$, such as monotonicity or the number and locations of optima, the standard strategy is to train a model in a federated and DP manner on every hyperparameter candidate and select the best one. However, this strategy suffers from severe privacy constraints: The privacy budget must be divided among iterations, roughly scaling with the square root of the number of iterations~\cite{dwork2006differential}. Since differentially private federated algorithms already consume a large privacy budget to maintain utility~\cite{geyer2017differentially}, even few iterations lead to an impractically small budget due to composition. Moreover, not every loss function can be published out-of-the-box in a private way, as a bounded sensitivity is mandatory. Even if a bound is enforced, e.g., through clipping, it does not automatically guarantee a meaningful privacy-utility trade-off. Furthermore, private selection algorithms like \cite{liu2019private} or its extension \cite{papernot2021hyperparameter}, which allow the privacy accounting to be independent of the number of iterations, cannot be applied either. This is because they share the assumption that an adversary does not know when the algorithm has stopped and whether this was due to random stopping or because the number of iterations or a threshold was reached.
Consequently, straightforward iterative hyperparameter search cannot be directly applied in federated and privacy-preserving settings.

\subsection{Approach} 

Evaluating each hyperparameter on the entire data set in a federated setting is infeasible under strict privacy constraints. To address this, we relax the formulation in \Cref{eq:globalargmin} and instead consider the following objective:
\begin{equation}
\label{eq:localevals}
\hpset_{j^*} = \underset{\hpset_j \in \hpset}{\argmax} \ |\{ i \in [\numclients] \mid \hpset_j = \underset{\hpset_t \in \hpset}{\argmin} \ \loss(\dataset_i, \hpset_t) \}|.
\end{equation}
Clients evaluate each hyperparameter $\hpset_t \in \hpset$ locally on their data set $\dataset_i$ and select the hyperparameter that reaches a minimum loss such that $\hpset_{j^*}$ is the hyperparameter most clients choose locally.
Throughout this work, we refer to those hyperparameters as good. As we show later, optimizing \Cref{eq:localevals} instead of \Cref{eq:globalargmin} yields strong privacy amplification and produces high-quality hyperparameters that closely align with those obtained by optimizing the global objective. A related idea was introduced in~\cite{mitic2025privacy}, where hyperparameters were locally tuned and averaged by the server, but that approach did not analyze any privacy implications.

A closer inspection of \Cref{eq:localevals} reveals a key advantage regarding privacy: we no longer require training a global model or computing the global loss. The only information needed is whether clients independently favor a particular hyperparameter. This naturally transforms the problem into a voting task, where clients act as voters and hyperparameters represent the candidates. This perspective provides an avenue to leverage well-established techniques from federated, differentially private voting~\cite{feddpvoting}.

Voting offers several notable benefits in this setting. First, differential privacy is naturally supported: if each client is restricted to vote for at most $\numlocaltops$ hyperparameters, then the L$2$-sensitivity, defined as the maximum change in the output due to replacing one client’s data, only grows with $\sqrt{\numlocaltops}$. Second, communication overhead remains low, since clients only need to send a scalar per candidate. However, a naive implementation would still require allocating privacy budget separately for each candidate, which is prohibitive when the number of hyperparameters is large. To overcome this, we combine voting with the privacy accounting framework of~\cite{feddpvoting,kairouz2021distributed}, which allows us to account for the aggregated voting result as a whole rather than per candidate. This results in additional privacy amplification, since the privacy cost depends only on the number of votes $\numlocaltops$ and not on the total number of hyperparameters. The final ingredient is secure aggregation: The privacy guarantees from voting and special accounting hold only if no individual votes are revealed to the server. To ensure this, we employ a secure summation protocol such as~\cite{secsum_bell} that computes the aggregated vote counts per hyperparameter without exposing individual contributions.

\subsection{DP-Hype}

We now introduce our algorithm \tool{}, presented in \Cref{alg::main}, which implements a federated and privacy-preserving hyperparameter search. The algorithm receives as input the federated data set $\dataset$, the set of hyperparameters $\hpset = \{\hpset_1, \dots, \hpset_\numhps\}$, a loss function $\loss: \mathcal{\dataset} \times \mathcal{\hpset} \rightarrow \mathbb{R}$, privacy parameters $\priveps,\privdelta$, and the number of local votes $\numlocaltops$. Its output is a single hyperparameter $\hpset_{j^*} \in \hpset$ that obtains the largest number of noisy votes aggregated across clients. The procedure consists of two phases: a client phase and a server phase. In the client phase, each client $i$ evaluates all hyperparameters $\hpset_j$ on their local data $\dataset_i$ to compute the losses $(\loss(\dataset_i, \hpset_1), \dots, \loss(\dataset_i, \hpset_\numhps))$. A local voting vector $\votingvec_i$ of length $\numhps$ is initialized with zeros, and the entries corresponding to the $\numlocaltops$ smallest losses are set to one. To ensure differential privacy, each local voting vector is perturbed by Gaussian noise scaled by $\noisemultiplier$, yielding a noisy version $\tilde{\votingvec}_i$. The noise scale $\noisemultiplier$ is based on $\numlocaltops,\priveps,\privdelta$ and can be determined using~\Cref{th::rdpToAdp}. After this step, clients jointly execute a secure summation protocol with the server, which ensures that only the aggregated noisy vote vector $\tilde{\votingvec}$ is revealed, while no individual votes are exposed. Since the sum of Gaussian random variables is Gaussian, the added noise guarantees differential privacy. Finally, the server selects the hyperparameter $\hpset_{j^*}$ with the maximum number of aggregated votes.  

All $\numlocaltops$ local votes are equally weighted and are not ranked on purpose. A ranking would either increase sensitivity or smaller ranked votes would quickly drown in noise. We empirically show the downsides of such weighted voting in \Cref{app:rawlosses}. Locally evaluating each hyperparameter introduces computational overhead for clients, but in \Cref{app:parabaseline} we show that this overhead is mandatory to achieve good performance.

\tool{} is agnostic to the specific learning task, requiring only that local and global evaluations of a loss function are available.
For ease of presentation, we describe it as minimizing the loss, but the procedure can equivalently be applied to maximization. Both the set of hyperparameters and their order are assumed to be public, enabling standard search strategies such as random or grid search by reordering candidates. Importantly, the local learning process does not need to satisfy differential privacy and can be performed without any modification, thereby avoiding unnecessary privacy-utility trade-offs. Noise is only introduced when information leaves the clients, in the voting stage. Moreover, communication overhead is minimal, as \tool{} needs only a single invocation of secure summation, making it lightweight in practice. 

\begin{algorithm}[t]
\caption{\tool{}}
\label{alg::main}
\SetKw{KwTo}{in}\SetKwFor{For}{for}{\string:}{}
\KwData{Data set $\dataset = \dataset_1 \cup \dots \cup \dataset_\numclients$, hyperparameter set $\hpset = \{\hpset_1, \dots, \hpset_\numhps\}$, loss function $\loss: \mathcal{D} \times \mathcal{H} \rightarrow \mathbb{R}$, DP-parameters $\priveps, \privdelta$, number of local votes $\numlocaltops$}
\KwResult{$\hpset_{j^*} \in \hpset$}
\nonl \textbf{Clients:} \\
Determine $\noisemultiplier$ based on $\priveps,\privdelta,\numlocaltops$ (\Cref{th::rdpToAdp}). \\
\For{$i$ \KwTo $[\numclients]$}{
    \For{$j$ \KwTo $[\numhps]$}{
        $\mathrm{loss}_j = \loss(\dataset_i, \hpset_j)$
    }
    $\votingvec_i = (0, \dots, 0) \in \mathbb{R}^\numhps$ \\
    \While{$||\votingvec_i||_1 < \numlocaltops$}{
    $j^* = \underset{j \in [\numhps] \land  \votingvec_i[j] = 0}{\argmin} \mathrm{loss}_j$ \\
        $\votingvec_i[j^*] = 1$ 
    }
    $\tilde{\votingvec}_i = \votingvec_i + \vb{z}, \quad \vb{z} \sim \mathcal{N}(0, \frac{\noisemultiplier^2}{\numclients}\mathbf{I}_\numhps)$ \\
}
\nonl \textbf{Server:} \\
$\tilde{\votingvec} = (\sum_{i=1}^\numclients \tilde{\votingvec}_i[1], \dots, \sum_{i=1}^\numclients \tilde{\votingvec}_i[\numhps])$ with SecSum \\
$\hpset_{j^*} = \underset{j \in [\numhps]}{\argmax} \ \tilde{\votingvec}[j]$ \\
\Return $\hpset_{j^*}$
\end{algorithm}

The modular design of \tool{} makes it independent of any specific secure aggregation protocol, as long as it supports summation. This reduces the deployment overhead of \tool{} and allows selecting a protocol that fits the individual requirements best.  A discussion of why we do not rely on a specific protocol can be found in \Cref{sec:discussion}. Possible extensions for \tool{} regarding malicious clients or server and robustness are moved to \Cref{sec::practicalcons} due to space limitations.

\section{Privacy Proof}
\label{sec::privacyproof}
In this section, we prove that \tool{} satisfies client-level $(\alpha, \privrdpeps)$-RDP. Guarantees for client-level $(\priveps, \privdelta)$-DP can then be derived using~\Cref{th::rdpToAdp}. Before starting with the proof, we give an overview of the proof structure.

\subsection{Proof Outline}
Denote \tool{} by $\algo$, to prove that $\algo$ satisfies RDP we have to bound the Rényi divergence $\infdiv{\algo(\dataset)}{\algo(\dataset')} \leq \privrdpeps$ for some $\alpha$ and $\privrdpeps$, and for any pair of neighboring data sets $\dataset, \dataset'$. From an algorithmic point of view, $\algo$ outputs a single hyperparameter $\algo(D) = \hpset_{j^*} \in \hpset$. However, from the perspective of an adversary, in our scenario the server, the output consists of much more. Since clients use secure summation to vote for each hyperparameter individually, the adversary actually learns $\algo(D) = \tilde{\votingvec} \in \mathbb{R}^\numhps$.  

So when analyzing the privacy guarantees of our algorithm, we have to take into account that the adversary can observe the entire noisy voting vector $\tilde{\votingvec}$, i.e., the voting result for every single hyperparameter. To avoid doing the accounting for every single hyperparameter and thereby run into composition, we consider the vector as a whole. The Gaussian mechanism allows us to do this if we consider the L$2$ sensitivity. This is the maximum change in the L$2$ norm of the vector that can happen in the worst case regarding all possible pairs of neighboring data sets. 

\subsection{Proof}
We start by analyzing the L$2$ sensitivity of the voting vector $\votingvec$ without any noise and show that due to the nature of the binary voting system and a fixed number of $\numlocaltops$ votes per client $\ltwosens\votingvec=\sqrt{2\numlocaltops}$. 
\begin{lemma}[L$2$-Sensitivity of $\votingvec$]
    Given the fixed number of local votes $\numlocaltops$, then the voting vector $\votingvec$ without any noise has L$2$-sensitivity $\ltwosens \votingvec = \sqrt{2\numlocaltops}$. 
\end{lemma}
\begin{proof}
    Let $\votingvec \in \mathbb{R}^\numhps$ be the aggregated voting vector without any noise and denote \tool{} as $\algo$. Further, let $\numhps$ denote the number of hyperparameters, $\numclients$ the number of clients and let $\votingvec_i[j] \in \{0, 1\}$ be the individual vote of client $i$ for hyperparameter $j$ before adding noise. In abuse of notation, we write $\votingvec, \votingvec'$ when based on $D,D'$, respectively.  
    
    \begin{align*}
        \ltwosens \votingvec &= \max_{\dataset \sim \dataset'} || \algo(\dataset) - \algo(\dataset') ||_2 
        \\
        &= \max_{\dataset \sim \dataset'} || \votingvec - \votingvec' ||_2 
        =  \max_{\dataset \sim \dataset'} \sqrt{\sum_{j=1}^\numhps \sum_{i=1}^\numclients (\votingvec_i[j] - \votingvec'_i[j])^2} 
        \\
        &\overset{(a)}{=} \max_{\dataset \sim \dataset'} \sqrt{\sum_{j=1}^\numhps (\votingvec_u[j] - \votingvec'_u[j])^2} 
        \\ 
        &\overset{(b)}{=} \max_{\dataset \sim \dataset'} \sqrt{\sum_{j=1}^{2\numlocaltops} (\votingvec_u[j] - \votingvec'_u[j])^2}
        \\
        &\leq \sqrt{\numlocaltops \cdot (-1)^2 + \numlocaltops \cdot (1)^2} = \sqrt{2\numlocaltops}
    \end{align*}
    $(a)$ By the definition of neighboring data sets, the entire data set $\dataset_u$ of a single client $u$ can be replaced.
    $(b)$ As for each client's voting vector $\votingvec$ it holds that $||\votingvec_i||_1 = k$, $2\numlocaltops$ entries in $\votingvec$ change at maximum.  Without loss of generality, we can assume that those entries are the first $2\numlocaltops$. The last inequality comes from the fact that clients can only choose between $0$ and $1$ for each hyperparameter and that in the worst case, $\numlocaltops$ ones flip to zeros and $\numlocaltops$ zeros to ones. 
\end{proof} 
Based on the fact that the output of \tool{} has a bounded L$2$-sensitivity, we can now prove that \tool{} satisfies RDP.
\begin{theorem}
Given $\numlocaltops \in \mathbb{N}$ and let $\alpha > 1$ then \tool{} denoted as $\algo$ satisfies client-level $(\alpha, \privrdpeps)$-RDP with
\[
 \privrdpeps = \infdiv{\algo(\dataset)}{\algo(\dataset')} = \frac{\alpha \numlocaltops}{\sigma^2}
\] for variance $\sigma^2$ and all neighboring data sets $\dataset,\dataset'$.
\end{theorem}
\begin{proof}
    The proof follows directly from \Cref{def:multvarigaussian} by setting $\ltwosens = \sqrt{2\numlocaltops}$.
\end{proof}
The noise scale $\noisemultiplier$ used in line $9$ of \Cref{alg::main} can then be determined based on \Cref{th::rdpToAdp}. 
To obtain $(\priveps,\privdelta)$-DP, \Cref{th::rdpToAdp} can be used.

\section{Utility Analysis}
\label{sec::utility}

In this section, we first quantify the probability of \tool{} achieving a good compromise and then go into detail about the parameter $\numlocaltops$ by simulating clients under various scenarios. 

\subsection{Selecting a Good Hyperparameter}

When considering the set of hyperparameters provided as input to \tool{}, we first note that most learning tasks do not admit a single hyperparameter that performs well while all others perform poorly. Instead, there typically exists a subset of candidates that yield good utility, though this subset is not too large, otherwise hyperparameter search would be trivial and random guessing would suffice. The objective of \tool{} is therefore to select one of these good hyperparameters such that a compromise is achieved across a majority of clients. However, because differential privacy requires that each client perturbs their local voting vector $\votingvec_i$, there remains the possibility that the added noise results in selecting a bad hyperparameter, particularly under tight privacy budgets. To formalize this, we denote by $\hpset_\mathrm{good} \subseteq \hpset$ the set of hyperparameters with good utility, and by $\hpset_\mathrm{bad} \subseteq \hpset \setminus \hpset_\mathrm{good}$ the set of bad ones. Importantly, it is not necessary to return exactly the hyperparameter with the most votes; any hyperparameter with a similar vote count is also acceptable, as it is likely to yield a comparable compromise. We therefore define the gap between the minimum number of votes for any good and the maximum number of votes for any bad hyperparameter based on $\votingvec$, the aggregated votes without noise, as
\begin{equation}
\label{eq:gap}
\gamma \coloneq \min_{\hpset_i \in \hpset_\mathrm{good}} (\votingvec[i]) - \max_{\hpset_j \in \hpset_\mathrm{bad}} (\votingvec[j]).
\end{equation}
Here, the $i$th component of $\votingvec$ corresponds to the total number of votes for hyperparameter $i$ without noise.
Based on this definition, the utility of \tool{}, denoted by $\algo$, can be expressed as $\Pr[\algo(\dataset) \in \hpset_\mathrm{good}]$, i.e., the probability that the algorithm outputs a good hyperparameter, when conditioning on the aggregated votes $\votingvec$. Building on prior work on private federated voting~\cite{feddpvoting}, we adapt their analysis and obtain the following result:

\begin{theorem}[Selecting a Good Hyperparameter]
\label{th:succprob}
Given a data set $\dataset$, a set of hyperparameters $\hpset$ containing a subset of bad candidates $\hpset_\mathrm{bad}$, and defining $\gamma$ as in \Cref{eq:gap}, then conditioned on the aggregated votes, \tool{}, denoted as $\algo$, with noise scale $\noisemultiplier$ outputs a good hyperparameter with probability
\begin{equation*}
\Pr[\algo(\dataset) \in \hpset_\mathrm{good}] \geq 1 - \frac{|\hpset_\mathrm{bad}|\noisemultiplier}{\gamma \sqrt{\pi}} \exp(-\frac{\gamma^2}{4\noisemultiplier^2}).
\end{equation*}
\end{theorem}
\begin{proof} 
The proof follows by modeling the aggregated noisy voting vector as $\tilde{\votingvec} = \votingvec + \vb{z}$ with $\vb{z} \sim \mathcal{N}(0,\noisemultiplier^2\mathbf{I}_\numhps)$, leveraging the guarantees of the secure summation protocol. If \tool{} outputs a bad hyperparameter, then $\vb{z}[i] - \vb{z}[j^*] \geq \gamma$ for some $\hpset_i \in \hpset_\mathrm{bad}$ and $j^* = \argmin_{\hpset_j \in \hpset_{\mathrm{good}}} \votingvec[j]$. Using the fact that $(\vb{z}[i] - \vb{z}[j^*]) \sim \mathcal{N}(0,2\noisemultiplier^2)$ and using a standard Gaussian tail bound, we get $1 - \Phi_{2\noisemultiplier^2}(\gamma) \leq \frac{\noisemultiplier}{\gamma\sqrt{\pi}}e^{-\frac{\gamma^2}{4\noisemultiplier^2}}$ where $\Phi$ denotes the CDF of the normal distribution with variance $2\noisemultiplier^2$. Applying a union bound over all bad candidates concludes the proof. 
\end{proof}

Importantly, the event that the noise overcomes $\gamma$ at least once is a necessary condition that must occur whenever \tool{} outputs a bad hyperparameter, though it is not a sufficient condition on its own. Thus, we can safely use this event in the proof to upper-bound the failure probability and derive a lower bound for \tool{}'s success probability. 
An immediate consequence of this analysis is that the number of clients does not affect the noise magnitude due to secure summation. Furthermore, the probability of \tool{} selecting a good hyperparameter grows exponentially with the gap $\gamma$, which matches intuition: When the separation between good and bad hyperparameters is large, noise unlikely reverses the outcome, even under limited privacy budgets. Conversely, the success probability decreases as the number of bad hyperparameters grows, since each additional candidate introduces another chance for noise to shift the outcome incorrectly. Note that $\noisemultiplier$ is directly influenced by $\numlocaltops$ as $\numlocaltops$ contributes to the sensitivity of $\votingvec$, which is why the trade-off between privacy and utility is also visible indirectly in this bound. 

\subsection{The Effect of the Maximum Number of Local Votes}
We now turn to the role of the parameter $\numlocaltops$ in \tool{}. Intuitively, $\numlocaltops$ influences the utility of the algorithm but, as shown in \Cref{sec::privacyproof}, it also directly correlates with privacy leakage, thereby introducing a privacy–utility trade-off. In the case of iid data distributions, the choice of $\numlocaltops$ has little impact, since a majority of clients will naturally agree on a specific hyperparameter, and a single vote per client would suffice. In such settings, the gap $\gamma$ from \Cref{eq:gap} is typically close to $\numclients$, which ensures robustness even under small privacy budgets. The situation becomes more interesting in non-iid settings, where distinct subpopulations of clients may prefer different hyperparameters. In this case, a hyperparameter representing a compromise might not receive the highest number of votes but could still perform well overall, for example by being the second- or third-best choice for many clients. The parameter $\numlocaltops$ is designed to capture precisely these scenarios by allowing clients to vote for multiple good candidates rather than only their top choice. However, the effectiveness of this mechanism depends heavily on the underlying data distributions and client-specific losses, making it difficult to provide general theoretical guarantees.

To illustrate the impact of $\numlocaltops$, we perform simulations under controlled conditions. We assume that the local losses of good and bad hyperparameters follow normal distributions with means $\mu_1=0$ and $\mu_2=1$, respectively, and equal standard deviation $\sigma_{\mathrm{loss}}$. Each client draws $|\hpset|$ local losses, $|\hpset_\mathrm{good}|$ good and $|\hpset_\mathrm{bad}|$ bad ones. The federated voting part of \tool{} is executed on the local losses without local training. After $5000$ repetitions, we estimate the success probability from \Cref{th:succprob} by counting how often \tool{} selects a hyperparameter from $\hpset_\mathrm{good}$ and divide it by the total number of repetitions. The experimental setup consists of $\numclients=250$ clients in total and varying values of $\numlocaltops$. In addition, we vary the privacy budget $\priveps \in \{0.25,1\}$ while fixing $\privdelta = 10^{-5}$. We choose such small privacy budgets on purpose as most effects of \tool{} like the privacy-utility trade-off only become visible for $\priveps \leq 1$. These simulations aim to provide empirical evidence of how $\numlocaltops$ impacts the privacy-utility trade-off in iid and non-iid regimes. 

\paragraph{IID vs. Non-IID}

\begin{figure}[t]
\centering

\subfloat[{\textbf{Influence of $\mathbf{\sigma_{\mathrm{loss}}}$ for $\mathbf{|\hpset_\mathrm{good}|=5}$ and $\mathbf{|\hpset|=100}$}}]{
  \includegraphics[width=\linewidth]{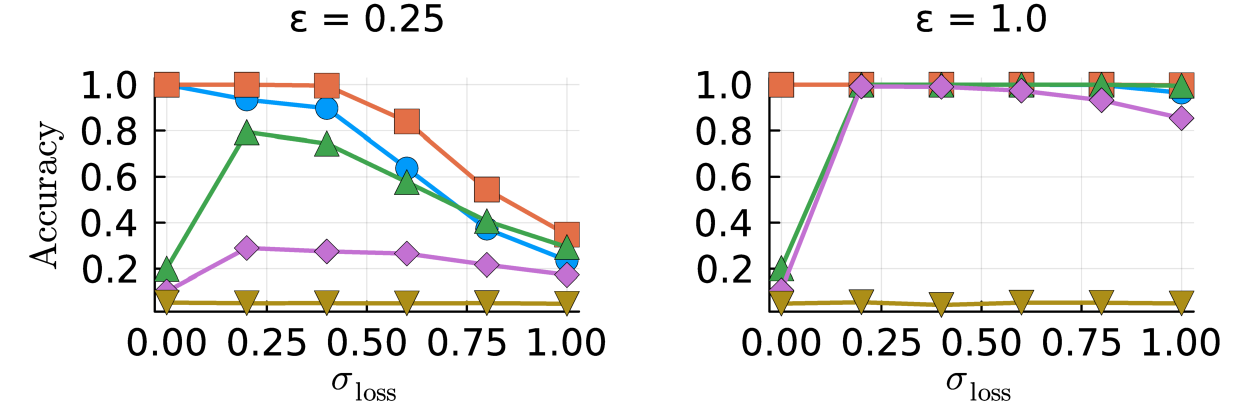}
  \label{fig:ksim_sigma}
}

\par

\subfloat[{\textbf{Influence of $\mathbf{|\hpset_\mathrm{good}|}$ for $\mathbf{\sigma_{\mathrm{loss}}=0.2}$ and $\mathbf{|\hpset|=100}$}}]{
  \includegraphics[width=\linewidth]{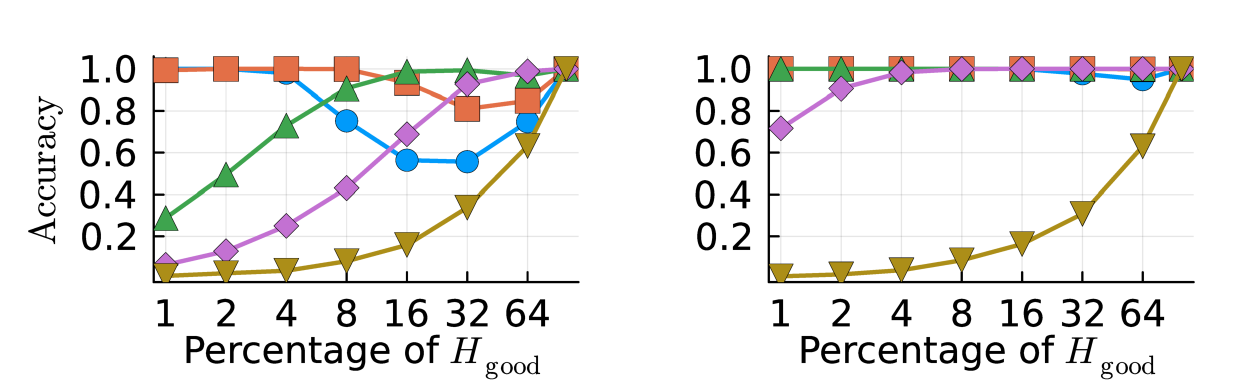}
  \label{fig:ksim_hgood}
}

\par

\subfloat[{\textbf{Influence of $\mathbf{|\hpset|}$ for $\mathbf{|\hpset_\mathrm{good}|=5}$ and $\mathbf{\sigma_{\mathrm{loss}}=0.2}$}}]{
  \includegraphics[width=\linewidth]{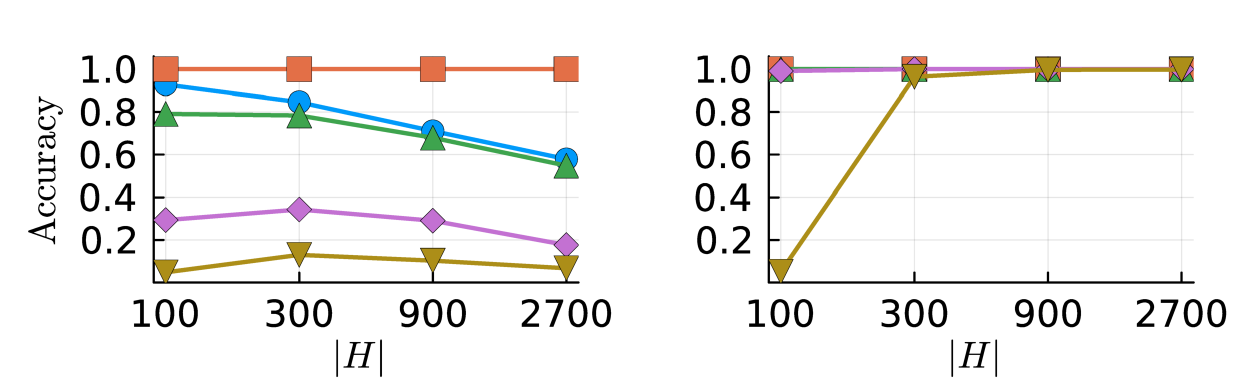}
  \label{fig:ksim_numhps}
}

\par

\subfloat{
  \includegraphics[width=\linewidth]{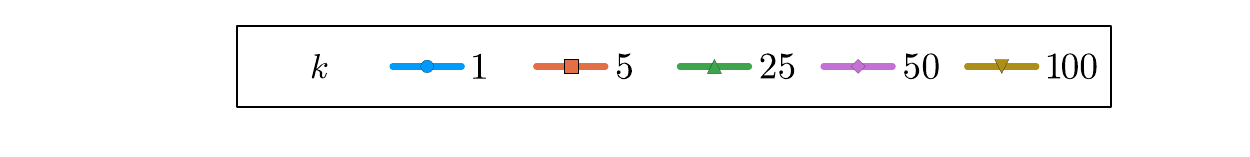}
}

\caption{Simulating (a) a varying overlap between good and bad losses with $|\hpset_\mathrm{good}|=5,|\hpset|=100$, (b) a varying percentage of good hyperparameters with $\sigma_{\mathrm{loss}}=0.2,|\hpset|=100$ and (c) a varying number of hyperparameters in total with $|\hpset_\mathrm{good}|=5,\sigma_{\mathrm{loss}}=0.2$. All under varying privacy budgets $\priveps$ with $\privdelta=10^{-5}$ and $\numclients=250$. We sample $|\hpset_\mathrm{good}|$ and $|\hpset_\mathrm{bad}|$ local losses ($|\hpset|$ in total) from $\mathcal{N}(0,\sigma_{\mathrm{loss}}^2)$ and $\mathcal{N}(1,\sigma_{\mathrm{loss}}^2)$, respectively. To show the impact of $\numlocaltops$ on utility, we measure the accuracy of \tool{}, i.e., how often it selects a loss pointing to a good hyperparameter.
A varying $\sigma_{\mathrm{loss}}$ creates different degrees of overlapping of both loss distributions, which uncovers the advantage of $\numlocaltops>1$. Varying $\hpset_\mathrm{good}$ reveals that smaller $\numlocaltops$ are favorable in case of a small number of good hyperparameters. Lastly, increasing the set of hyperparameters while leaving $|\hpset_\mathrm{good}|$ fixed, shows that the privacy-utility trade-off is only marginally influenced by the number of hyperparameters.}
\label{fig:ksim}
\end{figure}

We begin by studying the effect of varying the degree of non-iid, while deferring the analysis of different sample sizes and class distributions per client to the evaluation section. To simulate this setting, we increase the variance $\sigma_\mathrm{loss}^2$ of the local loss distributions gradually, causing the losses of good and bad hyperparameters to overlap more. In other words, the losses of good and bad hyperparameters become harder to distinguish, which effectively simulates scenarios where clients do not vote unanimously. We fix the number of good hyperparameters to $|\hpset_\mathrm{good}|=5$, and the results are reported in \Cref{fig:ksim_sigma}. The experiments show that larger values of $\numlocaltops$ improve accuracy when the overlap between good and bad hyperparameters increases. However, as $\numlocaltops$ grows too large, the privacy–utility trade-off becomes apparent: although a higher $\numlocaltops$ may capture more compromise solutions, it also induces stronger noise and forces clients to vote for bad hyperparameters when good ones are scarce. This effect is most pronounced for $\numlocaltops=100$, where accuracy drops nearly to zero since only five hyperparameters are good and their votes are overwhelmed by those for bad candidates. Importantly, the privacy–utility trade-off manifests only for very small values of $\priveps$, which demonstrates that \tool{} can still achieve high accuracy under tight privacy budgets.

\paragraph{Fraction of Good Hyperparameters}
Next, we investigate how the number of good hyperparameters influences the performance of \tool{}. To this end, we fix $\sigma=0.2$ to simulate a more practical scenario without only unanimous votes and vary the number of good hyperparameters as $|\hpset_\mathrm{good}| = 2^i$ for $i \in [7]$. The results are shown in \Cref{fig:ksim_hgood}. The findings confirm that the privacy–utility trade-off becomes relevant only for very small values of $\priveps$. Furthermore, for $\numlocaltops=100$ the accuracy does not exceed random guessing (appearing as a flat curve rather than a diagonal due to the log scale), which is expected since clients are forced to vote for all hyperparameters, leaving the server’s argmax decision dominated by noise. For intermediate choices of $\numlocaltops$, the privacy–utility trade-off becomes visible. In addition, larger choices of $\numlocaltops$ only perform well for a growing $|\hpset_\mathrm{good}|$. This is because if the number of good hyperparameters is much smaller than $\numlocaltops$ then clients are forced to vote for bad hyperparameters. 
Also note the dip observable around $\numlocaltops = 1$ with $|\hpset_\mathrm{good}| > 4$ and slightly smaller for $\numlocaltops = 5$ with $|\hpset_\mathrm{good}| > 16$. If there are a lot of good hyperparameters then the votes are distributed among these candidates, which in turn decreases the maximum votes for a single hyperparameter and thus decreases accuracy. The accuracy only increases again when bad hyperparameters are becoming rarer.

\paragraph{Total Number of Hyperparameters}
Lastly, we study how the total number of hyperparameter candidates affects performance by varying $|\hpset| \in \{100, 300, 900, 2700\}$ while keeping $|\hpset_\mathrm{good}|=5$ and $\sigma_\mathrm{loss}=0.2$ fixed. The results are shown in \Cref{fig:ksim_numhps}. For privacy budget $\priveps=1$, only $\numlocaltops=100$ is noticeably affected, because forcing clients to vote across essentially all candidates becomes close to random guessing, yielding an expected accuracy near $5/100=0.02$. For smaller privacy budgets, a smaller $\numlocaltops$ is preferable since fewer votes require less noise, and the results clearly indicate that the privacy–utility trade-off is largely independent of the total number of hyperparameters. Moreover, for $\numlocaltops>5$ the accuracy initially worsens up to $|\hpset|=100$, because many clients spend a large fraction of their votes on the same few bad candidates, increasing the probability that a bad candidate is selected. This effect vanishes as $|\hpset|$ grows and votes for bad candidates become more dispersed. The remaining slight downward trend in accuracy for larger $|\hpset|$ is due to an increased chance that, as more local losses are sampled, some bad candidate attains a loss close to that of a good candidate, making local ranking and thus voting more error-prone. Interestingly, if $\numlocaltops$ aligns with $|\hpset_\mathrm{good}|$ then the accuracy does not decrease at all. 
However, beyond some point, as $|\hpset|$ grows, it becomes increasingly likely that a poorly performing configuration attains a bad loss that falls within the range of the good losses. As a result, the accuracy of $\tool{}$ eventually decreases.

\paragraph{Selecting $\numlocaltops$}
\tool{} is deliberately designed to minimize the number of its own hyperparameters, since in differentially private hyperparameter search each additional parameter would require tuning, consuming part of the privacy budget and leaving less available for the actual search and learning task. Our simulations show that the choice of $\numlocaltops$ becomes critical only in very small privacy budget regimes and that choosing a $\numlocaltops$ that is slightly off the potential best choice does not break \tool{} but only decreases utility slightly. \Cref{fig:ksim} clearly shows that as long as $\numlocaltops$ is small \tool{} has a high probability of selecting a good hyperparameter. 

\section{Evaluation}
\label{sec::eval}
In this section, we provide further insights on how \tool{} is implemented in the federated learning framework Flower, and present our empirical evaluation on benchmark data sets in iid and non-iid scenarios. These experiments are focused on the privacy-utility trade-off and the resource consumption of \tool{}. Additional ablations are moved to \Cref{app:rawlosses,app:featherscomp,app:parabaseline} due to space limitations.

\subsection{Flower Implementation}
In general, Flower~\cite{beutel2020flower} simulates each client as a separate process, eliminating the need for managing many physical devices for clients. The server performs multiple iterations, and in each iteration a so-called strategy manages the configuration of clients, calls clients when their participation is needed, and processes the clients' results. Each client locally trains on their local share of data and sends their results to the server.
We implement \tool{} as such a strategy, a submodule in Flower, and we implement clients to be compatible with other learning tasks in addition to the ones we evaluate here.
\tool{} outputs a hyperparameter configuration that is then passed to the subsequent federated learning task. Inside Flower, \tool{} is implemented completely using PyTorch for model training. The source code of our implementation is publicly available at \href{https://github.com/UzL-PrivSec/dp-hype}{\url{https://github.com/UzL-PrivSec/dp-hype}}.

\subsection{Experimental Setup}

\paragraph{Data Sets and Learning Tasks}

We use the following data sets to evaluate the performance of \tool{}.
The MNIST~\cite{lecun2010mnist} data set which contains gray-scale images of handwritten digits, Cifar-$10$~\cite{krizhevsky2009learning}, also an image data set, but with three color channels and images containing objects and animals, and Adult~\cite{data1996adult}, which is of tabular form with a mix of categorical and continuous census attributes. An overview of the data sets and their properties is shown in \Cref{tbl::datasets}.
For MNIST we use a small convolutional model and for Cifar-$10$ a slightly larger one, to maintain a reasonable balance between accuracy and training time. In both cases, data are standardized, and the models are trained on a classification task to differentiate between $10$ classes.
For Adult, the data is also standardized, categorical attributes are one-hot-encoded and the attribute income is mapped into a binary attribute indicating whether income is above $50,000$. We then train a small MLP with dropout to learn the binary classification task for the attribute income. We refer to~\Cref{app:architectures} for details on model architectures. In all learning tasks, we use a train-test-split, such that $20\%$ of the data are used for testing and the rest for training. The training was performed on an NVIDIA A100-20C GPU with $20$GB VRAM by using the simulation module of Flower to train clients in parallel.

\paragraph{Hyperparameters}
We use a specific set of hyperparameters to mirror a realistic setup for training a neural network via SGD, we focus on the three parameters learning rate ($10$ values), learning rate decay ($5$ values) and momentum ($2$ values). The cross-product of candidates for all the parameters results in a total number of $100$ different hyperparameters. An overview of the individual candidates can be seen in \Cref{tbl::hps}. Smaller learning rates are chosen on purpose to challenge \tool{} with a lot of bad hyperparameters that attain model accuracy at best as good as guessing. The remaining learning rates in combination with the other parameters lead to a few excellent hyperparameters and a medium-sized number of mediocre performing hyperparameters.  We explicitly avoid adopting the same setup as prior work, because oftentimes the hyperparameter set consists of only a few candidates in total. On the one hand, this is a priori beneficial for privacy-preserving algorithms as the number of potential compositions remains low and on the other hand, a large fraction of good candidates makes randomly guessing the right hyperparameter a viable strategy.
\begin{table}[t]
\caption{An overview of the three evaluated data sets including the number of data points used for training, the number of attributes and the number of classes.}
\label{tbl::datasets}
\centering
\begin{tabular}{@{}lccc@{}}
\toprule
\textbf{Data Set} & \textbf{\# Data Points} & \textbf{Data Point Size} & \textbf{\# Classes} \\
\midrule
MNIST~\cite{lecun2010mnist} & $60{,}000$ & $28\times28$ (grayscale) & $10$ \\
Cifar-10~\cite{krizhevsky2009learning} & $60{,}000$ & $32\times32$ (RGB) & $10$ \\
Adult~\cite{data1996adult} & $\approx 50{,}000$ & 14 (mixed attributes) & $2$ \\
\bottomrule
\end{tabular}
\end{table}

\begin{table}[t]
\caption{The set of hyperparameters to be evaluated is the cross-product of $10$ learning rate, $5$ learning rate decay and $2$ momentum candidates, resulting in $100$ candidates in total.}
\label{tbl::hps}
\centering
\begin{tabular}{@{}ll@{}}
\toprule
\textbf{Hyperparameter} & \textbf{Candidates} \\
\midrule
Learning rate & $\{0.5,\,0.1,\,0.05,\,1\times10^{-3},\,5\times10^{-3},\,1\times10^{-5},$ \\
& $1\times10^{-6},\,5\times10^{-6},\,5\times10^{-7},\,1\times10^{-7}\}$ \\
Learning rate decay & $\{0.0,\,0.1,\,0.25,\,0.99,\,1.0\}$ \\
Momentum & $\{0,\,0.9\}$ \\
\addlinespace
\midrule
\textbf{Total combinations} & $\mathbf{10\cdot5\cdot2=100}$ \\
\bottomrule
\end{tabular}
\end{table}

\paragraph{Evaluated Algorithms}
As a first non-private baseline, which we call \opt{}, we investigate the accuracy of the best-performing hyperparameter.
For every hyperparameter candidate we train a model using federated learning and then select the best-performing hyperparameter as the optimal baseline. As another baseline and for lower bounding expected results of our \tool{}, we also report the average of accuracies over all candidates of hyperparameters which we call \randguess{}. This average represents the expected accuracy one would get from randomly selecting a hyperparameter, a strategy that does not have any privacy leakage since it is independent of sensitive data.
We then present all candidates to \tool{} and allow it to find a good compromise across all clients in a federated manner. The clients' local loss function is set to the standard accuracy metric of their model. As for the maximum number of local votes, we set it to $\numlocaltops=5$ as we observe that it works well across all data sets, so it does not need to be tuned based on the specific data set. So every client votes for the five hyperparameter candidates attaining the highest accuracy on their local model. With the hyperparameter that attains the most noisy votes, we then train the global model in a federated manner as well to assess the performance of the hyperparameter choice of \tool{}. For performance metrics, we report the average test accuracy and the corresponding $95\%$ confidence interval resulting from $20$ runs. Every federated training is based on FedAvg~\cite{fedavg} and the number of local as well as federated training runs is set to $5$ with a batch size of $64$. Except for Cifar-$10$ where $10$ epochs are performed locally and globally. We specifically avoid private training, using DP-SGD \cite{abadi2016deep} for instance, to evaluate the performance of hyperparameters, as this would add much more noise to the training pipeline and introduce additional bias to the results. Also, this allows us to better control a priori that a certain number of hyperparameters results in bad model performance to further challenge our algorithm. 

\begin{figure}[t]
    \centering
    \includegraphics[width=1.0  \linewidth]{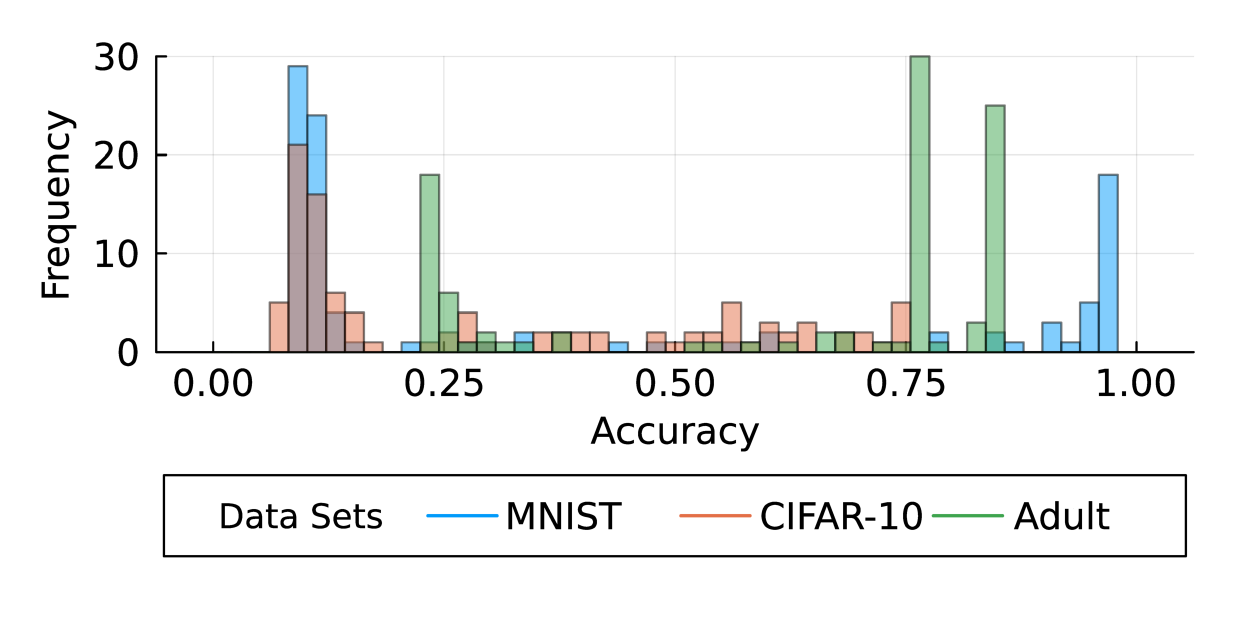}
    \caption{Histogram over individual accuracy values when training a global model on each data set using SGD on the set of hyperparameters shown in \Cref{tbl::hps}. As expected, the accuracies vary drastically depending on the data set.}
    \label{fig:lossdist}
\end{figure}

\begin{figure*}[t]
\centering

\subfloat[{\textbf{MNIST data set}}]{
  \includegraphics[width=0.95\textwidth]{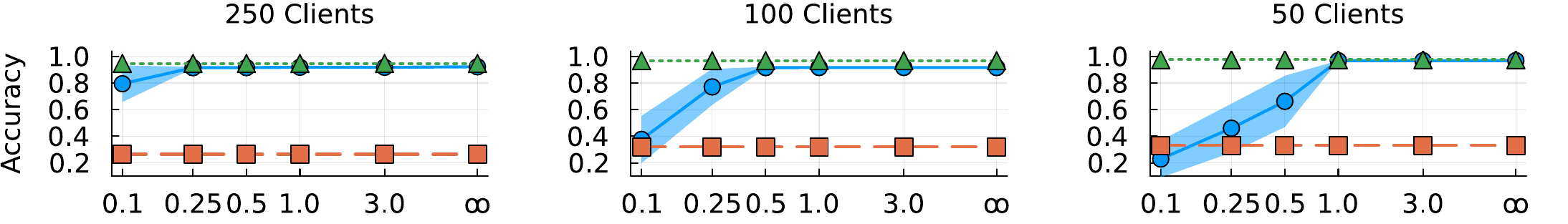}
  \label{fig:evaliidmnist}
}

\par

\subfloat[{\textbf{Cifar-$\mathbf{10}$ data set}}]{
  \includegraphics[width=0.95\textwidth]{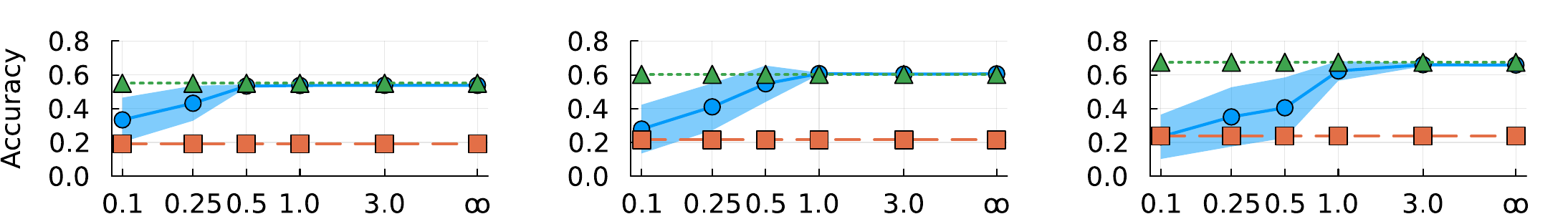}
  \label{fig:evaliidcifar10}
}

\par

\subfloat[{\textbf{Adult data set}}]{
  \includegraphics[width=0.95\textwidth]{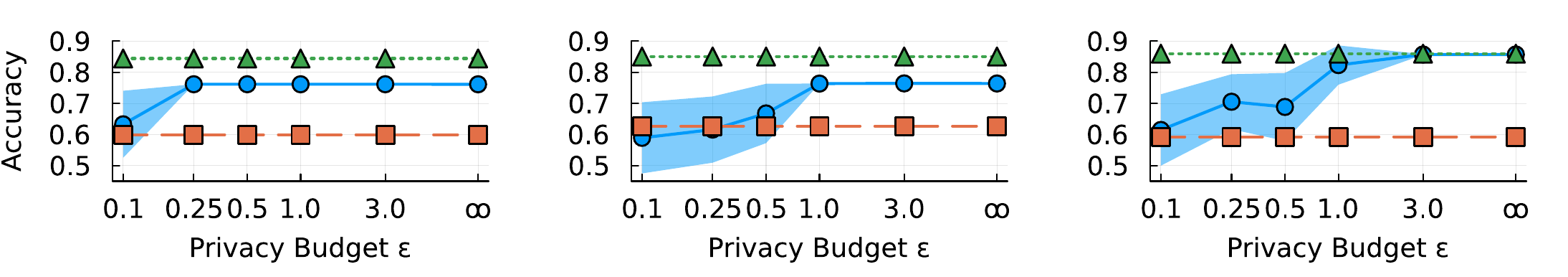}
  \label{fig:evaliidadult}
}

\par

\subfloat{
  \includegraphics[width=0.95\textwidth]{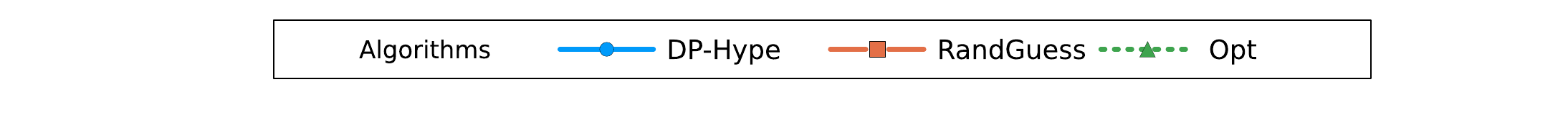}
}

\caption{Empirical results in terms of Accuracy for evaluating \tool{}, \randguess{} and \opt{} on the data sets MNIST, Cifar-10 and Adult for a varying number of clients and different privacy budgets. It can be seen that the set of hyperparameters contains a lot of bad candidates as \randguess{} is far away from \opt{}. Yet, \tool{} is very close to the best possible result even for small privacy budgets and this effect gets even stronger with a growing number of clients.}
\label{fig:evaliid}
\end{figure*}

Concerning a comparison with related work, the privacy\hyp preserving federated hyperparameter tuning that comes closest to our work is Feathers \cite{seng2022feathers}. Feathers selects hyperparameters based on loss differences between hyperparameter candidates. A closer inspection reveals that Feathers does not preserve differential privacy. While it adds noise to the loss values, it does so without applying clipping, which is necessary when dealing with unbounded quantities. Moreover, Feathers performs multiple iterations that are not included in its privacy accounting. 
Accounting for all iterations severely increases the leakage, making it impossible for Feathers to maintain competitive utility under strict privacy budgets. We compare \tool{} against a version of Feathers where we corrected the privacy accounting in \Cref{app:featherscomp}.

\begin{figure*}[t]
\centering

\subfloat[{\textbf{MNIST data set}}]{
  \includegraphics[width=0.95\textwidth]{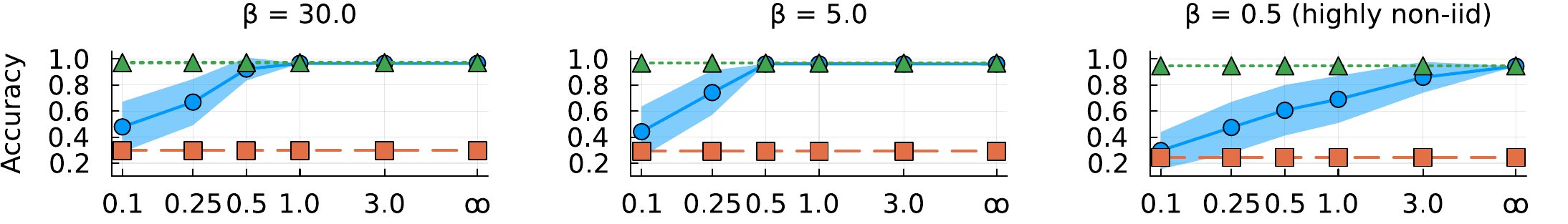}
  \label{fig:evalnoniidmnist}
}

\par

\subfloat[{\textbf{Cifar-$\mathbf{10}$ data set}}]{
  \includegraphics[width=0.95\textwidth]{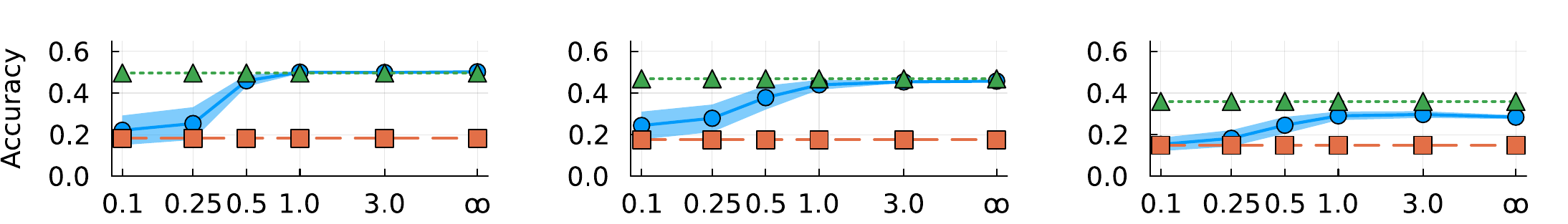}
  \label{fig:evalnoniidcifar10}
}

\par

\subfloat[{\textbf{Adult data set}}]{
  \includegraphics[width=0.95\textwidth]{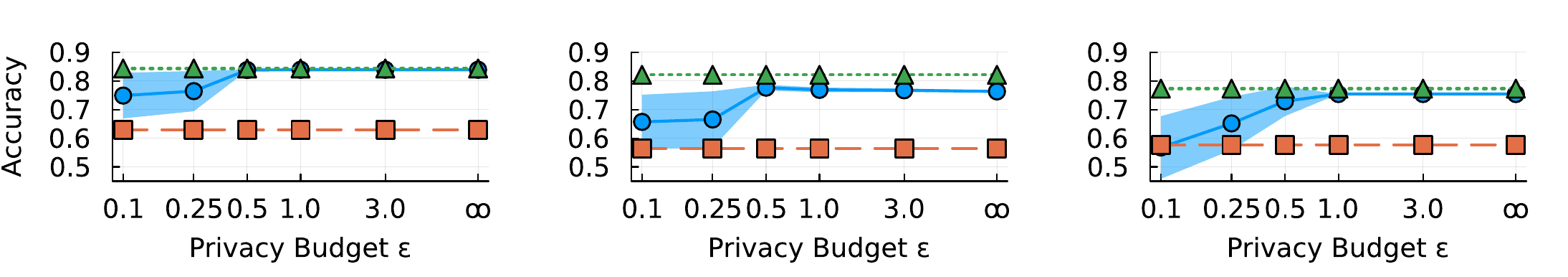}
  \label{fig:evalnoniidadult}
}

\par

\subfloat{
  \includegraphics[width=0.95\textwidth]{pets26/imgs/eval_iid_Legend_allN.pdf}
}

\caption{Empirical results in terms of Accuracy for evaluating \tool{}, \randguess{} and \opt{} on the data sets MNIST, Cifar-10 and Adult for $\numclients=100$ and different privacy budgets in the non-iid scenario. The degree of non-iid, denoted as $\dirichletparam$, increases from left to right, which means that the data heterogeneity increases as well. On MNIST, \tool{} again reaches peak utility, which decreases slightly for the harder $\dirichletparam=0.5$ scenario. On Cifar-$10$ the effect of non-iid scenarios is clear, as the utility of \opt{} shows. Still, \tool{} can uphold a high degree of utility even for those scenarios.}
\label{fig:evalnoniid}
\end{figure*}

\subsection{Privacy-Utility Trade-Off}
\label{sec:eval_privutiltradeoff}
To capture the privacy-utility trade-off of \tool{}, we evaluate the performance in iid as well as non-iid scenarios of \randguess{}, \opt{}, and \tool{} on three benchmark data sets: MNIST, Cifar-$10$, and Adult. For the number of clients, we consider $\numclients \in \{50, 100, 250\}$. For privacy, we fix $\privdelta = 10^{-5}$ while varying the privacy budget $\priveps \in \{0.1, 0.25, 0.5, 1, 3, \infty\}$ resulting in noise scales for the Gaussian mechanism $\sigma^2~\in~\{103, 46, 24, 12.5, 4.7, 0\}$, respectively.

In the iid scenario, each client is assumed to have local training data sampled from the same underlying distribution. To construct this setup, we randomly divide the training data into equal portions such that every client receives the same amount of data and the distribution of labels remains nearly identical across clients. The results are summarized in \Cref{fig:evaliid}.
For non-iid scenarios, clients differ both in the size of their local training data and in the distribution of labels. We simulate such scenarios using a Dirichlet partitioning approach. Specifically, for each label $l$ in the data set, we draw proportions from an $\numclients$-dimensional Dirichlet distribution with concentration parameter $\dirichletparam$, which determines how samples with label $l$ are split across clients. Usually, the letter $\alpha$ is used in the literature for non-iid scenarios, but we use $\dirichletparam$ to avoid a conflict with $\alpha$ from Rényi DP. Once these proportions are fixed, the data set is shuffled and federated accordingly. Smaller values of $\dirichletparam$ result in highly skewed distributions, potentially assigning nearly all samples of a given label to a single client, while larger values of $\dirichletparam$ yield more iid-like partitions. This setup aligns with prior work and effectively combines two common non-iid conditions: label skewness and variation in client data sizes. For our experiments, we investigate $\dirichletparam \in \{30.0, 5.0, 0.5\}$. The results are in~\Cref{fig:evalnoniid} as well as in~\Cref{app:evalnoniid}. We visualize label distribution and size of clients' local data sets in \Cref{app:visualnoniid}. In this setting, all models are trained globally for $5$ training runs.

\paragraph{IID Scenario}
First, we take a closer look at \Cref{fig:evaliid}. The results highlight two key effects: the positive impact of increasing the number of clients and the inherent privacy-utility trade-off. The accuracy improvement due to more clients can be explained by our utility analysis through the parameter $\gamma$, which measures the gap in votes between good and bad hyperparameters. In an iid setting, more clients amplify the votes for strong hyperparameters, making it harder for noise to distort the selection process. This effect is especially visible in the MNIST and Cifar-$10$ results, while on Adult, \tool{} occasionally favors second-best hyperparameters, though the trend remains consistent. As expected, decreasing the privacy budget lowers utility, but this loss can be partially compensated by increasing the number of clients, since more clients produce more votes and thus improve the signal-to-noise ratio.  

\paragraph{Non-IID}
When comparing \Cref{fig:evalnoniidN250,fig:evalnoniid,fig:evalnoniidN50}, increasing the number of clients reduces the number of data points available per client, which explains why the utility of all algorithms decreases. However, it also reduces the noise contributed per client, so \tool{} benefits from larger client populations. A higher degree of non-iidness further decreases utility overall and also changes the privacy–utility trade-off of \tool{}, because votes are spread across more hyperparameter candidates, diluting the evidence for any single choice. It is therefore natural that \tool{} cannot consistently match the optimal utility, as purely local evaluations cannot fully replace global evaluation. In an iid setting, FedAvg performs nearly as well as centralized training on the pooled data, whereas \tool{} may miss globally best candidates because they might not materialize in any client’s local evaluation and thus cannot be selected. In the non-iid setting, FedAvg no longer maintains the same utility and, in addition, higher variance across clients produces more strong candidates. Overall, these results indicate that trading global hyperparameter evaluation for local evaluation yields a substantial privacy improvement while sacrificing some optimality (which also carries over to utility), and that \tool{} remains effective under smaller privacy budgets than typically required in privacy-preserving federated learning tasks.

\sisetup{
  detect-weight = true,
  detect-family = true
}
\begin{table}[t]
\centering
\caption{
Resource consumption of \tool{} for the local hyperparameter evaluation and the federated voting. The former is measured per client individually, and captures the GPU vRAM overhead in MB and the runtime (RT) in seconds for each data set. The voting part shows the individual up- and download bandwidth in KB and the runtime in seconds including all clients. All measurements are averaged over five runs with $\numclients \in \{50,100,250\}$.
}
\label{tab:secagg_ress}
\begin{tabular}{ll S S cr}
\toprule
& \multicolumn{3}{c}{Local HP Evaluation} & \multicolumn{2}{c}{Federated Voting} \\
\cmidrule(lr){2-4}\cmidrule(lr){5-6}
 $\numclients$ & Data set & vRAM & RT & Bandwidth & {RT} \\
\midrule
\multirow{3}{*}{$50$} & MNIST  & 68.29 & 120.22 & \multirow{3}{*}{\commbx{33.78}{29.72}} & \multirow{3}{*}{$16.66$} \\
& Cifar-10 & 199.44 & 131.36 & & \\
& Adult  & 18.29  & 50.52 & &  \\
\midrule
\multirow{3}{*}{$100$} & MNIST  & 68.29  & 63.78 & \multirow{3}{*}{\commbx{61.02}{49.36}} & \multirow{3}{*}{$28.36$} \\
& Cifar-10 & 199.44 & 71.78 & & \\ 
& Adult  & 18.29  & 27.05 & & \\
\midrule
\multirow{3}{*}{$250$} & MNIST  & 68.29  & 26.53 & \multirow{3}{*}{\commbx{147.30}{111.38}} & \multirow{3}{*}{$193.13$} \\
& Cifar-10 & 199.44 & 29.51 & & \\
& Adult  & 18.31  & 12.74 & & \\
\bottomrule
\end{tabular}
\end{table}

\subsection{Resource Consumption}
We also evaluate the resource consumption of \tool{} in terms of GPU memory footprint, bandwidth and runtime. Overall, the computational cost decomposes naturally into two phases: (i) local training and evaluation performed independently by each client for every considered hyperparameter, and (ii) a federated, privacy-preserving voting step, which requires a single invocation of a secure summation protocol.

For the local hyperparameter search, we report the per-client GPU memory consumption (in MB) by using the PyTorch built-in function \emph{torch.cuda.max\_memory\_allocated},
and the wall-clock runtime (in seconds). For the voting phase of \tool{}, we use the \textsc{Flower} implementation of SecAgg++~\cite{secsum_bell}. The protocol exposes two main parameters: the number of shares (i.e., the number of secret shares created per client) and the reconstruction threshold, which specifies how many shares are required to reconstruct an encoded value. We set the number of shares to $0.3 \numclients$, which provides reasonable security. While a logarithmic number of shares would suffice in principle, for $\numclients < 1000$ this would amount to only a few shares. Since we do not consider client dropouts in this setting, we set the reconstruction threshold equal to the number of shares (dropout-tolerance is still possible with \tool{}, but is not evaluated here). We fix the voting vector size to $100$, consistent with the other experiments, and measure per-client upload and download traffic (in KB) as well as the overall runtime (in seconds). These experiments are conducted five times in the iid setting for $\numclients \in \{50, 100, 250\}$ across all three data sets, and the results are shown in~\Cref{tab:secagg_ress}.

We observe that both GPU RAM consumption and runtime are smallest for Adult, followed by MNIST, and highest for Cifar-10. This trend is expected, as both model complexity and input dimensionality increase in this order. Moreover, as the number of clients increases, the GPU RAM consumption remains constant while the overall runtime decreases substantially. The former follows directly because model complexity and batch size do not depend on $\numclients$. The latter is explained by the fact that data points are federated across clients. Hence, more clients imply fewer local training examples per client and, consequently, shorter local training times. In all configurations, the per-client GPU memory usage remains in the MB range and the end-to-end runtime stays within a few minutes. The measurements of the federated voting are independent of the data set, because the size of the voting vector does not change. As expected, SecAgg++ introduces a clear overhead. Nevertheless, the bandwidth overhead per client remains within KB, and the runtime increases only to slightly above three minutes even for $\numclients=250$. This efficiency stems from the fact that the voting phase communicates only a single vector per client.

\section{Related Work}
\label{sec::relatedwork}

\paragraph{Centralized DP Hyperparameter Search}
In the central setting, the trusted aggregator has access to the entire data set and the adversary can only observe the final output of an algorithm. 
The naive approach would be to pay some part of the privacy budget for each hyperparameter separately. The individual privacy budget can be calculated based on composition theorems. Although these theorems are becoming increasingly tight, the individual privacy budget still decreases roughly in $\sqrt{r}$ with $r$ compositions~\cite{mironov2017renyi}, which is just too small for most privacy-preserving federated learning algorithms~\cite{mcmahan2017learning}. 
There is a line of work that shows how hyperparameter search can be done more efficiently in DP. It starts with the sparse vector technique~\cite{dwork2006differential} that answers queries simply by outputting whether a noisy threshold exceeds the noisy query answer without any influence of the privacy budget on the number of queries. This is extended by~\cite{liu2019private} and the follow-up works~\cite{papernot2021hyperparameter,koskela2023practical}, which show that the search for hyperparameters, with and without a given threshold for utility, can be performed by only paying twice or three times the privacy budget as for a single hyperparameter. However, these techniques cannot be directly applied in a federated scenario. This is because the main privacy amplification comes from random stopping and strict assumptions, such as the adversary can only observe the final output and not intermediate results, which requires a trusted third party.

\paragraph{Federated DP Hyperparameter Search}
Conducting a hyperparameter search on sensitive data in the federated setting is far less studied in contrast to the central setting, as privacy-preserving federated learning by itself is a difficult task.
Some solutions like LDP-DSBO~\cite{chen2024locally} are specialized on specific learning algorithms, which leverage some properties of gradients used to train models. Others like PrivTuna~\cite{mitic2025privacy} focus on simulating the central setting with SMPC by paying the price of high communication overhead that only works with a few clients without real privacy protection as the results are just published without noise.
When it comes to a federated privacy-preserving hyperparameter search in general, there exist essentially two algorithms. First, DP-FTS-DE~\cite{dai2021differentially} for privacy-preserving federated Bayesian optimization and second, Feathers~\cite{seng2022feathers} combining neural architecture and hyperparameter search based on loss differences between hyperparameters. 
However, DP-FTS-DE tackles a different problem than \tool{}, as it aims to help clients with the convergence of local hyperparameters. Moreover, DP-FTS-DE relies on a trusted third party, and there is no guarantee that the hyperparameters of one client work for another client. Additionally, the privacy analysis does not account for the hyperparameters to be published.  
Feathers, on the other hand, comes closest to our approach. However, 
Feathers fails to preserve DP, because it uses the cross-entropy loss with additive noise without any clipping, even though this loss function is unbounded and could theoretically be drastically influenced when exchanging data of a single client. Furthermore, Feathers does not take into account that each hyperparameter that is evaluated increases the leakage. In addition, Feathers does not perform well under realistic assumptions, i.e., small privacy budgets and a set of hyperparameters containing a lot of bad-performing candidates~\cite{seng2022feathers}. 

We intentionally focus on classic hyperparameter search, in contrast to Feathers, which additionally conducts differentiable neural architecture search (NAS). 
NAS is typically studied as a distinct problem setting due to its discrete and structured search space, specialized algorithms, and substantially different compute–accuracy trade-offs~\cite{wang2024advances, benmeziane2024medical}. \tool{} can be extended to NAS treating architectural choices as hyperparameters or additionally conducting differentiable NAS.

\paragraph{Differentially Private Voting}
In the absence of a trusted aggregator or a secure aggregation protocol, local differential privacy (LDP) serves as the standard model to realize voting, requiring each client to enforce privacy individually. Several algorithms build on this notion: \cite{david2023local} evaluate binary voting with well-known mechanisms such as randomized response and RAPPOR; \cite{yan2019dpwevote} proposes a multi-stage protocol based on randomized response; and \cite{xu2024privacy} applies local votes for labeling unlabeled data. Despite their practicality, LDP mechanisms generally suffer from high noise: they require large client populations to average it out, and small privacy budgets can significantly reduce utility.
Similar to \tool{}, other work employs secure aggregation as well to simulate the central model. \cite{zhu2020federated} introduce noiseless voting for heavy hitters,~\cite{bagdasaryan2021towards} construct global histograms to identify frequently visited locations, and~\cite{feddpvoting} leverage differentially private voting to label unlabeled data points based on local contributions. On the other hand, \tool{} extends this voting approach with a top-$\numlocaltops$ mechanism, translates it towards a differentially private hyperparameter search and requires only a single invocation of the secure summation protocol.

\section{Discussion}
\label{sec:discussion}
In a federated setting where each party has a list of items, \emph{heavy hitter protocols} have the goal of selecting those items that occur most often. 
We could in principle deploy heavy hitter protocols to identify the most frequently chosen hyperparameters. Below, we discuss why we decided against deploying existing heavy hitter protocols. In \cref{sec::practicalcons}, we go into detail about the robustness of \tool{} and how malicious parties can be handled. 

We decided against local DP heavy hitter protocols~\cite{zhang2025heavyhitterldp, chadha2024heavyhitterldp, li2024heavyhitterldp, bun2019heavyhittersldp, wang2021heavyhitterldp,bassily2017heavyhittersldp}, to achieve a strong privacy-utility trade-off. LDP protocols do not rely on SMPC, but make the output of each client differentially private. Prior work~\cite{balle2019blanket} shows that they have an additional factor of $\sqrt{\numclients}$ privacy leakage in terms of $\varepsilon$.
In the shuffle model~\cite{balle2019blanket}, this overhead can be reduced, but it assumes a mixing protocol, which again introduces trust assumptions that impede deployment. 

There is a rich body of literature about SMPC-based heavy hitter protocols that rely on a small number (typically $2$ or $3$) of computation parties~\cite{damgard2024privateselection,balle2025heavyhitter,boneh2021heavyhitters,boehler2021heavyhitters} or SMPC-based private selection protocols~\cite{damgard2024privateselection}. 
As these protocols become very inefficient if the number of computation parties increases, this number has to be small in practice. In a multilateral deployment scenario with more than $10$ parties, agreeing on how to set up a handful of computation servers can become a point of friction that significantly delays deployment. In addition, usually at least one of the computation parties has to be a trusted third party.
So, we decided against relying on computation parties and designed \tool{} such that it solely relies on secure summation, as this is a much simpler requirement, for which scalable protocols exist.
However, the modularity of \tool{} allows replacing the secure aggregation part easily if the use of one of the above protocols is required.

\section{Conclusion}
\label{sec::conclusion}
In this work, we introduced \tool{}, a privacy-preserving algorithm to perform a federated hyperparameter search. \tool{} is designed to find a compromise among clients that minimizes the loss of a learning task locally as much as possible. The core idea is to evaluate hyperparameters locally on each client’s private data and let each client be part of a federated voting protocol with $\numlocaltops$ votes per client. Then, the hyperparameter with the most votes is chosen for the learning task. The aggregation of local votes is implemented using a secure summation protocol such that the server only sees the aggregated votes per hyperparameter. Each client adds a small amount of noise to their private voting vector locally such that, in sum, the noise suffices to preserve differential privacy.

Compared to previous work, \tool{} is agnostic to the specific learning task, preserves the strong notion of client-level differential privacy and its privacy guarantees are independent of the total number of hyperparameters.
In addition, we formally quantify the probability that \tool{} outputs a good hyperparameter. This probability converges exponentially to $1$ with a growing gap between votes for good and for bad hyperparameters. We implemented \tool{} in the popular federated learning framework Flower and evaluated it on three benchmark data sets. The results from a simulation and on the data sets demonstrate the high degree of utility that \tool{} is able to uphold even for small privacy budgets, and in iid as well as non-iid scenarios.

\section{Acknowledgement}
This research has been conducted within the AnoMed project (https://anomed.de/)
funded by the BMFTR (German Bundesministerium für Forschung, Technologie und Raumfahrt) and by the European Union – NextGenerationEU.

Throughout this work, we used generative AI-based tools to revise the text, improve flow and correct any typos, grammatical errors, and awkward phrasing. These tools were only used for small editorial purposes, and we inspected all outputs to ensure accuracy and originality. Importantly, all content is our own work and is not based on AI tools.

\bibliographystyle{IEEEtran}
\bibliography{main}

@inproceedings{secsum_bell,
  author       = {James Henry Bell and
                  Kallista A. Bonawitz and
                  Adri{\`{a}} Gasc{\'{o}}n and
                  Tancr{\`{e}}de Lepoint and
                  Mariana Raykova},
  title        = {{Secure Single-Server Aggregation with (Poly)Logarithmic Overhead}},
  booktitle    = {{CCS} '20: 2020 {ACM} {SIGSAC} Conference on Computer and Communications
                  Security, Virtual Event, USA, November 9-13, 2020},
  pages        = {1253--1269},
  publisher    = {{ACM}},
  year         = {2020},
}

@inproceedings{dwork2006differential,
  title={{Differential Privacy}},
  author={Dwork, Cynthia},
  booktitle={International colloquium on automata, languages, and programming},
  pages={1--12},
  year={2006},
  organization={Springer}
}

@inproceedings{mironov2017renyi,
  title={{R{\'e}nyi Differential Privacy}},
  author={Mironov, Ilya},
  booktitle={2017 IEEE 30th computer security foundations symposium (CSF)},
  pages={263--275},
  year={2017},
  organization={IEEE}
}

@inproceedings{balle2020hypothesis,
  title={{Hypothesis Testing Interpretations and Renyi Differential Privacy}},
  author={Balle, Borja and Barthe, Gilles and Gaboardi, Marco and Hsu, Justin and Sato, Tetsuya},
  booktitle={International Conference on Artificial Intelligence and Statistics},
  pages={2496--2506},
  year={2020},
  organization={PMLR}
}

@article{geyer2017differentially,
  title={{Differentially Private Federated Learning: A Client Level Perspective}},
  author={Geyer, Robin C and Klein, Tassilo and Nabi, Moin},
  journal={arXiv preprint arXiv:1712.07557},
  year={2017}
}

@article{mcmahan2017learning,
  title={{Learning Differentially Private Recurrent Language Models}},
  author={McMahan, H Brendan and Ramage, Daniel and Talwar, Kunal and Zhang, Li},
  journal={arXiv preprint arXiv:1710.06963},
  year={2017}
}

@article{papernot2021hyperparameter,
  title={{Hyperparameter Tuning with Renyi Differential Privacy}},
  author={Papernot, Nicolas and Steinke, Thomas},
  journal={arXiv preprint arXiv:2110.03620},
  year={2021}
}

@inproceedings{liu2019private,
  title={{Private Selection from Private Candidates}},
  author={Liu, Jingcheng and Talwar, Kunal},
  booktitle={Proceedings of the 51st Annual ACM SIGACT Symposium on Theory of Computing},
  pages={298--309},
  year={2019}
}

@article{koskela2023practical,
  title={{Practical Differentially Private Hyperparameter Tuning with Subsampling}},
  author={Koskela, Antti and Kulkarni, Tejas D},
  journal={Advances in Neural Information Processing Systems},
  volume={36},
  pages={28201--28225},
  year={2023}
}

@article{wirth2022easysmpc,
  title={{EasySMPC: A Simple but Powerful No-Code Tool for Practical Secure Multiparty Computation}},
  author={Wirth, Felix Nikolaus and Kussel, Tobias and M{\"u}ller, Armin and Hamacher, Kay and Prasser, Fabian},
  journal={BMC bioinformatics},
  volume={23},
  number={1},
  pages={531},
  year={2022},
  publisher={Springer}
}

@article{gamiz2025challenges,
  title={{Challenges and Future Research Directions in Secure Multi-Party Computation for Resource-Constrained Devices and Large-Scale Computations}},
  author={Gamiz, Idoia and Regueiro, Cristina and Lage, Oscar and Jacob, Eduardo and Astorga, Jasone},
  journal={International Journal of Information Security},
  volume={24},
  number={1},
  pages={1--29},
  year={2025},
  publisher={Springer}
}

@article{zhou2024secure,
  title={{Secure Multi-Party Computation for Machine Learning: A Survey}},
  author={Zhou, Ian and Tofigh, Farzad and Piccardi, Massimo and Abolhasan, Mehran and Franklin, Daniel and Lipman, Justin},
  journal={IEEE Access},
  year={2024},
  publisher={IEEE}
}

@inproceedings{hu2021make,
  title={{How to Make Private Distributed Cardinality Estimation Practical, and Get Differential Privacy for Free}},
  author={Hu, Changhui and Li, Jin and Liu, Zheli and Guo, Xiaojie and Wei, Yu and Guang, Xuan and Loukides, Grigorios and Dong, Changyu},
  booktitle={30th USENIX security symposium (USENIX Security 21)},
  pages={965--982},
  year={2021}
}

@article{mitic2025privacy,
  title={{Privacy-Preserving Hyperparameter Tuning for Federated Learning}},
  author={Mitic, Natalija and Pyrgelis, Apostolos and Sav, Sinem},
  journal={IEEE Transactions on Privacy},
  year={2025},
  publisher={IEEE}
}

@article{seng2022feathers,
  title={{Feathers: Federated Architecture and Hyperparameter Search}},
  author={Seng, Jonas and Prasad, Pooja and Mundt, Martin and Dhami, Devendra Singh and Kersting, Kristian},
  journal={arXiv preprint arXiv:2206.12342},
  year={2022}
}

@article{dai2021differentially,
  title={{Differentially Private Federated Bayesian Optimization with Distributed Exploration}},
  author={Dai, Zhongxiang and Low, Bryan Kian Hsiang and Jaillet, Patrick},
  journal={Advances in Neural Information Processing Systems},
  volume={34},
  pages={9125--9139},
  year={2021}
}

@inproceedings{chen2024locally,
  title={{Locally Differentially Private Decentralized Stochastic Bilevel Optimization with Guaranteed Convergence Accuracy}},
  author={Chen, Ziqin and Wang, Yongqiang},
  booktitle={Forty-first International Conference on Machine Learning},
  year={2024}
}

@article{beutel2020flower,
  title={{Flower: A Friendly Federated Learning Research Framework}},
  author={Beutel, Daniel J and Topal, Taner and Mathur, Akhil and Qiu, Xinchi and Fernandez-Marques, Javier and Gao, Yan and Sani, Lorenzo and Kwing, Hei Li and Parcollet, Titouan and Gusmão, Pedro PB de and Lane, Nicholas D},
  journal={arXiv preprint arXiv:2007.14390},
  year={2020}
}

@article{lecun2010mnist,
  title={{MNIST Handwritten Digit Database}},
  author={LeCun, Yann and Cortes, Corinna and Burges, CJ},
  journal={ATT Labs [Online]. Available: http://yann.lecun.com/exdb/mnist},
  volume={2},
  year={2010}
}

@article{krizhevsky2009learning,
  title={{Learning Multiple Layers of Features from Tiny Images}},
  author={Krizhevsky, Alex and Hinton, Geoffrey and others},
  year={2009},
  publisher={Toronto, ON, Canada}
}

@misc{data1996adult,
author       = {Becker, Barry and Kohavi, Ronny},
title        = {{Adult}},
year         = {1996},
howpublished = {UCI Machine Learning Repository},
note         = {{DOI}: \url{https://doi.org/10.24432/C5XW20}}
}

@InProceedings{fedavg,
  title = 	 {{Communication-Efficient Learning of Deep Networks from Decentralized Data}},
  author = 	 {McMahan, Brendan and Moore, Eider and Ramage, Daniel and Hampson, Seth and Arcas, Blaise Aguera y},
  booktitle = 	 {Proceedings of the 20th International Conference on Artificial Intelligence and Statistics},
  pages = 	 {1273--1282},
  year = 	 {2017},
  editor = 	 {Singh, Aarti and Zhu, Jerry},
  volume = 	 {54},
  series = 	 {Proceedings of Machine Learning Research},
  month = 	 {20--22 Apr},
  publisher =    {PMLR},
  url = 	 {https://proceedings.mlr.press/v54/mcmahan17a.html},
}

@article{feddpvoting,
  title={{Voting-Based Approaches for Differentially Private Federated Learning}},
  author={Zhu, Yuqing and Yu, Xiang and Tsai, Yi-Hsuan and Pittaluga, Francesco and Faraki, Masoud and Wang, Yu-Xiang and others},
  journal={arXiv preprint arXiv:2010.04851},
  year={2020}
}

@inproceedings{david2023local,
  title={{Local Differential Privacy in Voting.}},
  author={David, Bernardo and Giustolisi, Rosario and Mortensen, Victor and Pedersen, Morten},
  booktitle={ITASEC},
  year={2023}
}

@article{yan2019dpwevote,
  title={{DPWeVote: Differentially Private Weighted Voting Protocol for Cloud-Based Decision-Making}},
  author={Yan, Ziqi and Liu, Jiqiang and Liu, Shaowu},
  journal={Enterprise Information Systems},
  volume={13},
  number={2},
  pages={236--256},
  year={2019},
  publisher={Taylor \& Francis}
}

@inproceedings{zhu2020federated,
  title={{Federated Heavy Hitters Discovery with Differential Privacy}},
  author={Zhu, Wennan and Kairouz, Peter and McMahan, Brendan and Sun, Haicheng and Li, Wei},
  booktitle={International Conference on Artificial Intelligence and Statistics},
  pages={3837--3847},
  year={2020},
  organization={PMLR}
}

@article{bagdasaryan2021towards,
  title={{Towards Sparse Federated Analytics: Location Heatmaps under Distributed Differential Privacy with Secure Aggregation}},
  author={Bagdasaryan, Eugene and Kairouz, Peter and Mellem, Stefan and Gasc{\'o}n, Adri{\`a} and Bonawitz, Kallista and Estrin, Deborah and Gruteser, Marco},
  journal={arXiv preprint arXiv:2111.02356},
  year={2021}
}

@inproceedings{xu2024privacy,
  title={{Privacy-Preserving Heterogeneous Federated Learning for Sensitive Healthcare Data}},
  author={Xu, Yukai and Zhang, Jingfeng and Gu, Yujie},
  booktitle={2024 IEEE Conference on Artificial Intelligence (CAI)},
  pages={1142--1147},
  year={2024},
  organization={IEEE}
}

@inproceedings{kairouz2021distributed,
  title={{The Distributed Discrete Gaussian Mechanism for Federated Learning with Secure Aggregation}},
  author={Kairouz, Peter and Liu, Ziyu and Steinke, Thomas},
  booktitle={International Conference on Machine Learning},
  pages={5201--5212},
  year={2021},
  organization={PMLR}
}

@inproceedings{abadi2016deep,
  title={Deep learning with differential privacy},
  author={Abadi, Martin and Chu, Andy and Goodfellow, Ian and McMahan, H Brendan and Mironov, Ilya and Talwar, Kunal and Zhang, Li},
  booktitle={Proceedings of the 2016 ACM SIGSAC conference on computer and communications security},
  pages={308--318},
  year={2016}
}

@article{reddi2020adaptive,
  title={Adaptive federated optimization},
  author={Reddi, Sashank and Charles, Zachary and Zaheer, Manzil and Garrett, Zachary and Rush, Keith and Kone{\v{c}}n{\`y}, Jakub and Kumar, Sanjiv and McMahan, H Brendan},
  journal={arXiv preprint arXiv:2003.00295},
  year={2020}
}

@inproceedings{damgard2024privateselection,
author = {Damg\r{a}rd, Ivan and Keller, Hannah and Nelson, Boel and Orlandi, Claudio and Pagh, Rasmus},
title = {Differentially Private Selection from Secure Distributed Computing},
year = {2024},
isbn = {9798400701719},
publisher = {Association for Computing Machinery},
address = {New York, NY, USA},
url = {https://doi.org/10.1145/3589334.3645435},
doi = {10.1145/3589334.3645435},
booktitle = {Proceedings of the ACM Web Conference 2024},
pages = {1103–1114},
numpages = {12},
location = {Singapore, Singapore},
series = {WWW '24}
}

@INPROCEEDINGS {balle2025heavyhitter,
author = { Balle, Borja and Bell-Clark, James and Cheu, Albert and Gascon, Adria and Katz, Jonathan and Raykova, Mariana and Schoppmann, Phillipp and Steinke, Thomas },
booktitle = { 2025 IEEE Symposium on Security and Privacy (SP) },
title = {{ Hash-Prune-Invert: Improved Differentially Private Heavy-Hitter Detection in the Two-Server Model }},
year = {2025},
volume = {},
ISSN = {},
pages = {2903-2918},
doi = {10.1109/SP61157.2025.00186},
url = {https://doi.ieeecomputersociety.org/10.1109/SP61157.2025.00186},
publisher = {IEEE Computer Society},
}

@INPROCEEDINGS{boneh2021heavyhitters,
  author={Boneh, Dan and Boyle, Elette and Corrigan-Gibbs, Henry and Gilboa, Niv and Ishai, Yuval},
  booktitle={2021 IEEE Symposium on Security and Privacy (SP)}, 
  title={Lightweight Techniques for Private Heavy Hitters}, 
  year={2021},
  volume={},
  number={},
  pages={762-776},
  doi={10.1109/SP40001.2021.00048}
}

@inproceedings{boehler2021heavyhitters,
author = {B\"{o}hler, Jonas and Kerschbaum, Florian},
title = {Secure Multi-party Computation of Differentially Private Heavy Hitters},
year = {2021},
isbn = {9781450384544},
publisher = {Association for Computing Machinery},
address = {New York, NY, USA},
url = {https://doi.org/10.1145/3460120.3484557},
doi = {10.1145/3460120.3484557},
booktitle = {Proceedings of the 2021 ACM SIGSAC Conference on Computer and Communications Security},
pages = {2361–2377},
numpages = {17},
}

@article{zhang2025heavyhitterldp,
author = {Zhang, Yuemin and Ye, Qingqing and Hu, Haibo},
title = {Federated Heavy Hitter Analytics with Local Differential Privacy},
year = {2025},
issue_date = {February 2025},
publisher = {Association for Computing Machinery},
address = {New York, NY, USA},
volume = {3},
number = {1},
url = {https://doi.org/10.1145/3709739},
doi = {10.1145/3709739},
journal = {Proc. ACM Manag. Data},
month = feb,
articleno = {42},
numpages = {27},
}

@INPROCEEDINGS{chadha2024heavyhitterldp,
  author={Chadha, Karan and Chen, Junye and Duchi, John and Feldman, Vitaly and Hashemi, Hanieh and Javidbakht, Omid and McMillan, Audra and Talwar, Kunal},
  booktitle={2024 IEEE Conference on Secure and Trustworthy Machine Learning (SaTML)}, 
  title={Differentially Private Heavy Hitter Detection using Federated Analytics}, 
  year={2024},
  volume={},
  number={},
  pages={512-533},
  doi={10.1109/SaTML59370.2024.00032}
}

@article{li2024heavyhitterldp,
author = {Li, Xiaochen and Liu, Weiran and Lou, Jian and Hong, Yuan and Zhang, Lei and Qin, Zhan and Ren, Kui},
title = {Local Differentially Private Heavy Hitter Detection in Data Streams with Bounded Memory},
year = {2024},
issue_date = {February 2024},
publisher = {Association for Computing Machinery},
address = {New York, NY, USA},
volume = {2},
number = {1},
url = {https://doi.org/10.1145/3639285},
doi = {10.1145/3639285},
journal = {Proc. ACM Manag. Data},
month = mar,
articleno = {30},
numpages = {27},
}

@article{bun2019heavyhittersldp,
author = {Bun, Mark and Nelson, Jelani and Stemmer, Uri},
title = {Heavy Hitters and the Structure of Local Privacy},
year = {2019},
issue_date = {October 2019},
publisher = {Association for Computing Machinery},
address = {New York, NY, USA},
volume = {15},
number = {4},
issn = {1549-6325},
url = {https://doi.org/10.1145/3344722},
doi = {10.1145/3344722},
journal = {ACM Trans. Algorithms},
month = oct,
articleno = {51},
numpages = {40},
}

@ARTICLE{wang2021heavyhitterldp,
author={Wang, Tianhao and Li, Ninghui and Jha, Somesh},
journal={ IEEE Transactions on Dependable and Secure Computing },
title={{ Locally Differentially Private Heavy Hitter Identification }},
year={2021},
volume={18},
number={02},
ISSN={1941-0018},
pages={982-993},
doi={10.1109/TDSC.2019.2927695},
url = {https://doi.ieeecomputersociety.org/10.1109/TDSC.2019.2927695},
publisher={IEEE Computer Society},
address={Los Alamitos, CA, USA},
}

@inproceedings{bassily2017heavyhittersldp,
author = {Bassily, Raef and Nissim, Kobbi and Stemmer, Uri and Thakurta, Abhradeep},
title = {Practical locally private heavy hitters},
year = {2017},
isbn = {9781510860964},
publisher = {Curran Associates Inc.},
booktitle = {Proceedings of the 31st International Conference on Neural Information Processing Systems},
pages = {2285–2293},
numpages = {9},
series = {NIPS'17}
}

@InProceedings{balle2019blanket,
author="Balle, Borja
and Bell, James
and Gasc{\'o}n, Adri{\`a}
and Nissim, Kobbi",
editor="Boldyreva, Alexandra
and Micciancio, Daniele",
title="The Privacy Blanket of the Shuffle Model",
booktitle="Advances in Cryptology -- CRYPTO 2019",
year="2019",
publisher="Springer International Publishing",
pages="638--667",
}

@article{sabater2022fedavg,
author = {Sabater, C\'{e}sar and Bellet, Aur\'{e}lien and Ramon, Jan},
title = {An accurate, scalable and verifiable protocol for federated differentially private averaging},
year = {2022},
issue_date = {Nov 2022},
publisher = {Kluwer Academic Publishers},
address = {USA},
volume = {111},
number = {11},
issn = {0885-6125},
url = {https://doi.org/10.1007/s10994-022-06267-9},
doi = {10.1007/s10994-022-06267-9},
journal = {Mach. Learn.},
month = nov,
pages = {4249–4293},
numpages = {45},
}

@article{wang2024advances,
  title={Advances in neural architecture search},
  author={Wang, Xin and Zhu, Wenwu},
  journal={National Science Review},
  volume={11},
  number={8},
  pages={nwae282},
  year={2024},
  publisher={Oxford University Press}
}

@inproceedings{benmeziane2024medical,
  title={Medical neural architecture search: Survey and taxonomy},
  author={Benmeziane, Hadjer and Hamzaoui, Imane and Cherif, Zayneb and El Maghraoui, Kaoutar},
  booktitle={International joint conference on artificial intelligence},
  year={2024}
}

\begin{appendices}

\crefalias{section}{appendix}
\crefalias{subsection}{appendix}
\crefalias{chapter}{appendix} 
\Crefname{appendix}{Appendix}{Appendices}
\crefname{appendix}{appendix}{appendices}

\section{Practical Considerations}
\label{sec::practicalcons}
In this section, we state the advantages of a majority vote when dealing with skewed data or dropouts. Afterward, we discuss how the utility and privacy of \tool{} could be affected if the server or some clients are malicious and may deviate from the protocol. 

\subsection{Robustness}
For practical deployment, it is important to consider how \tool{} behaves under skewed data and client dropouts during a protocol run. Notably, \tool{} can also operate in a fully federated setting without a central server, depending on the underlying secure summation protocol. Since the algorithm itself does not require central coordination, each client can independently process the aggregated result once the final sum is available, and any client could take over the aggregation task without additional computational cost. With respect to data distribution, the majority-based design of \tool{} is both a strength and a limitation. On the one hand, the algorithm assumes that data is sufficiently similar among clients, ensuring that most clients reach comparable conclusions when evaluating hyperparameters. On the other hand, this same reliance on majority decisions makes \tool{} robust to a small fraction of clients with highly skewed data, as their influence on the global outcome is limited.

A second challenge arises from client dropouts, a common issue in federated learning where clients may disconnect during the protocol due to failures or unstable connections. Whether dropouts can be tolerated without disrupting the aggregation depends on the secure summation protocol, and dropout resistance is an important feature of many of such protocols. In the case of \tool{}, dropouts would normally lead to insufficiently noised voting vectors $\votingvec$, undermining differential privacy guarantees. A key advantage of \tool{} is that it can tolerate dropouts by design with only minimal adjustments. Specifically, if the algorithm is configured to tolerate up to a fraction $\xi$ of dropouts, each client replaces the denominator $\numclients$ in the variance scaling of the Gaussian noise with $(1-\xi)\cdot \numclients$. This ensures that even if up to $\xi \cdot \numclients$ clients drop out, the aggregated vector still receives sufficient noise. While this approach slightly worsens the privacy–utility trade-off, it provides a practical safeguard against unpredictable client failures and has also been applied in related work.

\subsection{Malicious Clients and Server}
So far, we have assumed that all clients and the server are honest-but-curious, meaning they follow the protocol faithfully, but may still try to infer sensitive information from observed messages. A natural extension is to ask how \tool{} behaves if some participants are malicious and actively deviate from the protocol. One straightforward attack against privacy would be for clients to avoid the addition of noise to their local voting vectors. This issue can be mitigated similarly to the dropout scenario by introducing additional noise to compensate for potentially malicious clients. In that case, even if a certain fraction of clients refuses to add noise, the remaining honest participants are sufficient to ensure differential privacy. Since \tool{} can operate successfully under small privacy budgets, this approach remains feasible and the effect of individual malicious clients becomes negligible as the total number of participants grows. Moreover, clients cannot observe the votes of others by design, which prevents them from coordinating such an attack.

Beyond privacy, malicious clients may also attempt to sabotage utility by voting randomly or refusing to support good hyperparameters. Fortunately, voting mechanisms naturally provide a degree of Byzantine robustness: The impact of any single malicious client is bounded by the sensitivity of their local voting vector $\votingvec_i$, preventing them from arbitrarily manipulating the result. Still, the use of secure summation means that individual votes are hidden, and thus protocol compliance cannot be verified directly. Consequently, the impact of a malicious client's local voting vector is potentially unbounded. Fortunately, we can make use of the modularity of \tool{} and its independence of specific secure aggregation protocols. By using GOPA~\cite{sabater2022fedavg}, a fully federated secure aggregation protocol, input poisoning can be mitigated as it supports input validation mechanisms in a scalable way.  

Finally, if the server itself is malicious, it could attempt to influence the protocol outcome by misreporting the result of the secure summation. This type of manipulation affects only utility, not privacy, as noise is applied locally on the client side. The only way for a malicious server to compromise privacy would be by tampering with the secure summation protocol to reveal individual votes. However, many secure aggregation schemes are designed to tolerate a fraction of malicious clients and even adversarial servers, ensuring that privacy guarantees remain intact under realistic threat models.

\section{Model Architectures}
\label{app:architectures}
For our evaluation, we use the following NN architectures:
\paragraph{MNIST} It begins with two convolutional layers featuring $32$ and $64$ filters, respectively, each using a $3 \times 3$ kernel. This is followed by a $2$d max-pooling layer of size $2$. The classifier head consists of two dense layers with a hidden size of $128$. The first dense layer utilizes ReLU activation and is preceded by a dropout layer with a rate of $0.25$. An additional dropout layer with a rate of $0.5$ is applied before the final output, which serves as input for the cross entropy loss.
\paragraph{Cifar-10} The architecture consists of two sequential convolutional blocks followed by a global average pooling layer and a final dense layer. The first block contains two convolutional layers with $32$ filters, each followed by batch normalization and ReLU activation, ending with a max-pooling layer and a dropout rate of $0.25$. The second block follows a similar structure but increases the filter count to $64$ and $128$, respectively, also concluding with max-pooling and $0.25$ dropout. A fully connected layer produces the output, which serves as input for the cross entropy loss.
\paragraph{Adult} For this data set, an MLP is employed, consisting of two hidden layers with $64$ units each. Both layers utilize ReLU activation functions. To mitigate overfitting, a dropout layer with a rate of $0.2$ is applied after the second hidden layer. The resulting raw logits are passed directly to a binary cross-entropy loss function.

\section{Visualization of Non-IID Clients}
\label{app:visualnoniid}

To better understand how tough the non-iid learning scenarios are and how much the scenarios we evaluate differ from one another, we visualize the label distribution and size of clients' local data sets. In the non-iid scenarios, the MNIST, Cifar-10 and Adult data sets are divided among clients using a Dirichlet partitioning approach. More specifically, for each label $l$ in the data set, we draw proportions from an $\numclients$-dimensional Dirichlet distribution with concentration parameter $\dirichletparam$, which determines how samples with label $l$ are split across clients. Once these proportions are fixed, the data set is shuffled and federated accordingly. When setting the parameter $\dirichletparam$ of the Dirichlet distribution to $30.0$, the data set is split more iid than non-iid. For smaller values of $\dirichletparam$ however, the label distribution is quite skewed over the clients and the size of their local data sets can vary a lot, see \Cref{fig:visualnoniidmnistcifar} and \Cref{fig:visualnoniidadult}.

\begin{figure*}[t]
\centering

\subfloat[{\textbf{MNIST data set}}]{
  \includegraphics[width=0.25\textwidth]{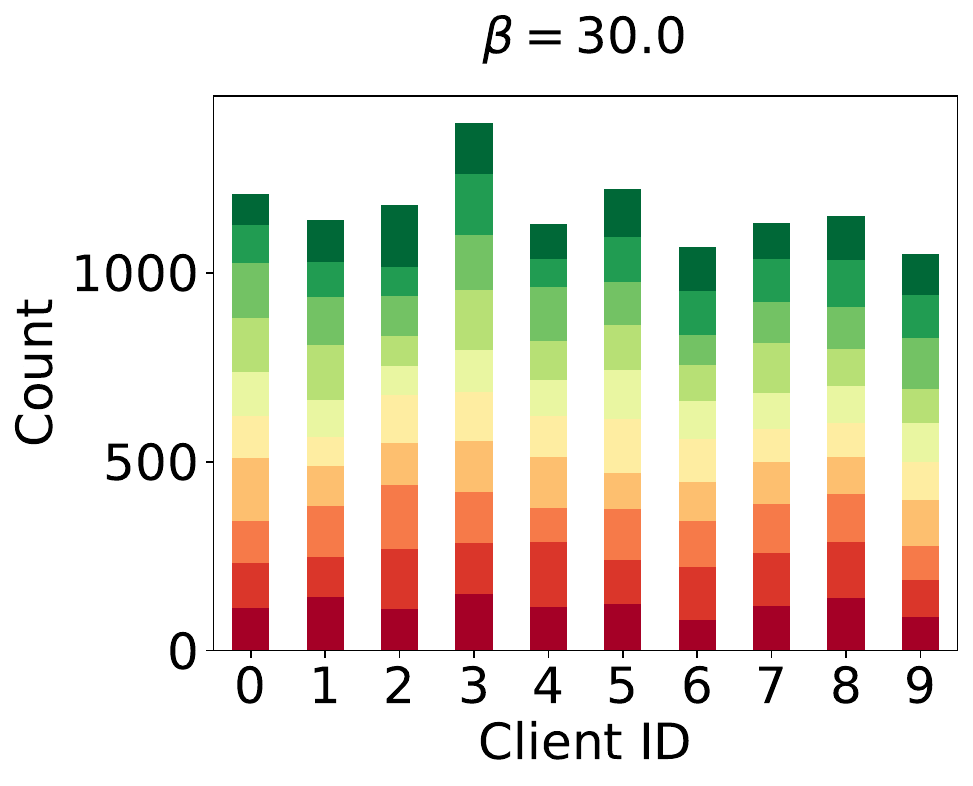}
  \includegraphics[width=0.25\textwidth]{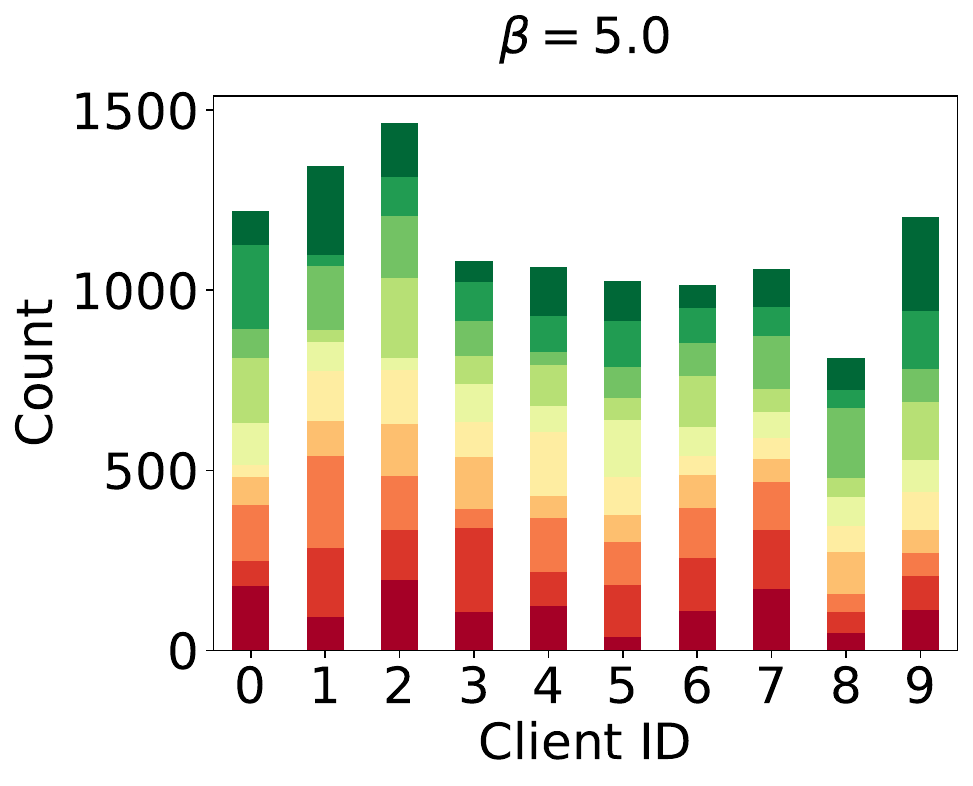}
  \includegraphics[width=0.25\textwidth]{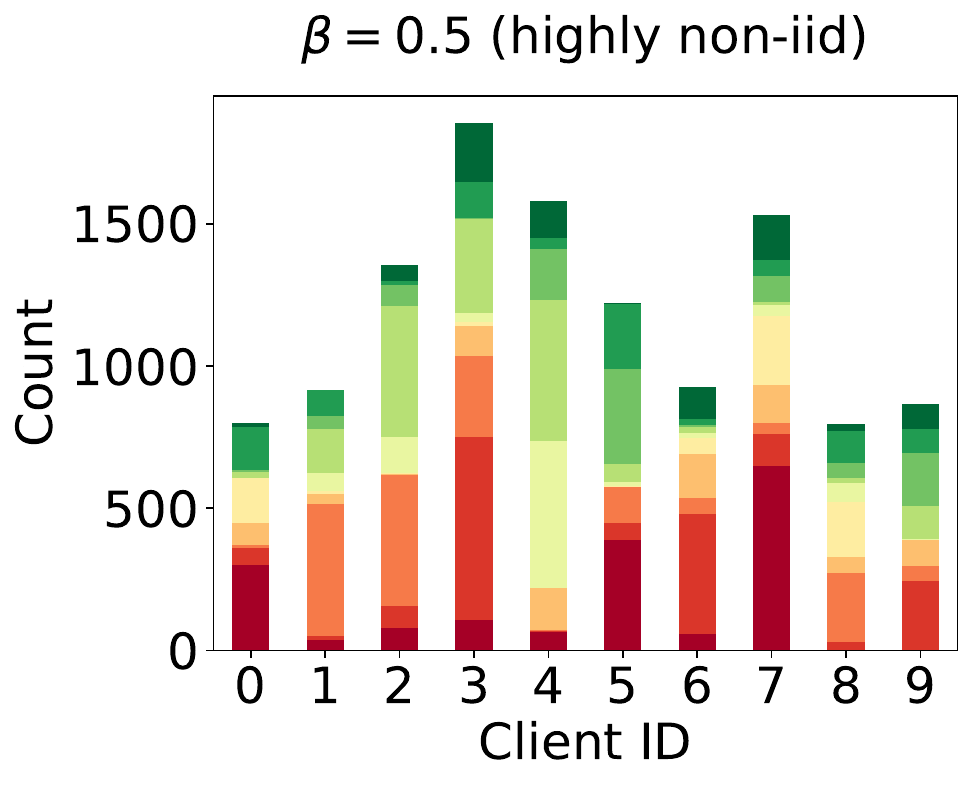}
  \label{fig:visual_mnist}
}

\par

\subfloat[{\textbf{Cifar-10 data set}}]{
  \includegraphics[width=0.25\textwidth]{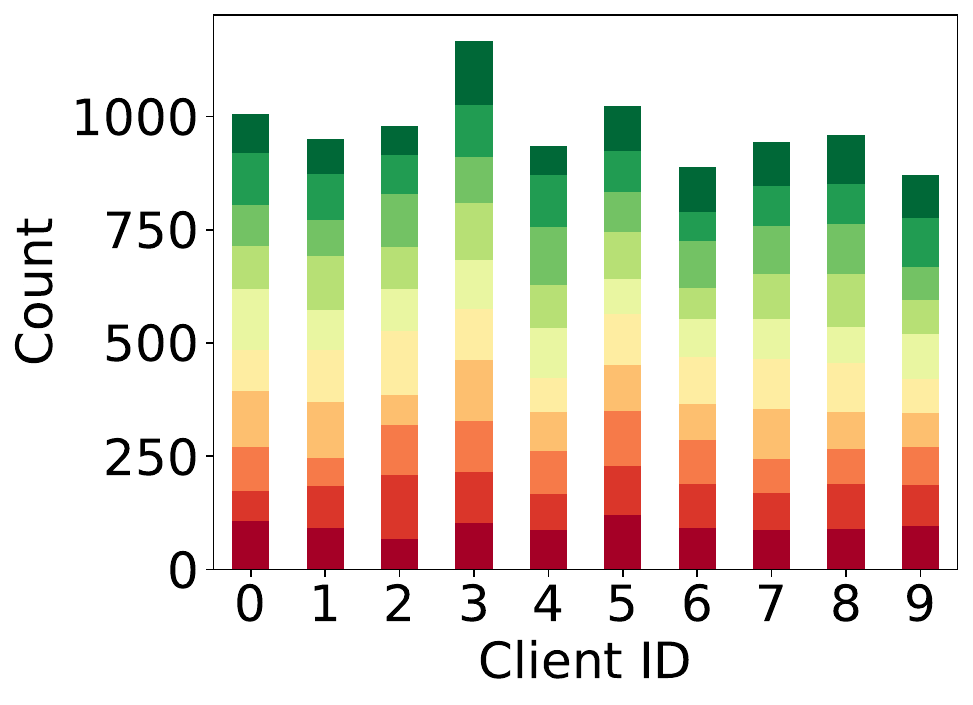}
  \includegraphics[width=0.25\textwidth]{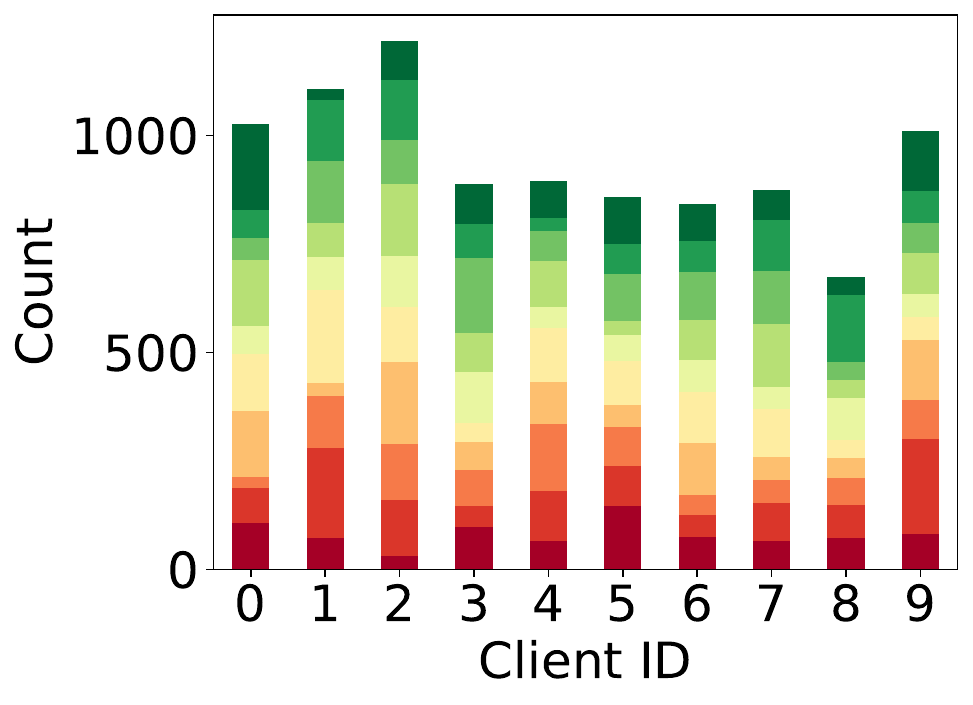}
  \includegraphics[width=0.25\textwidth]{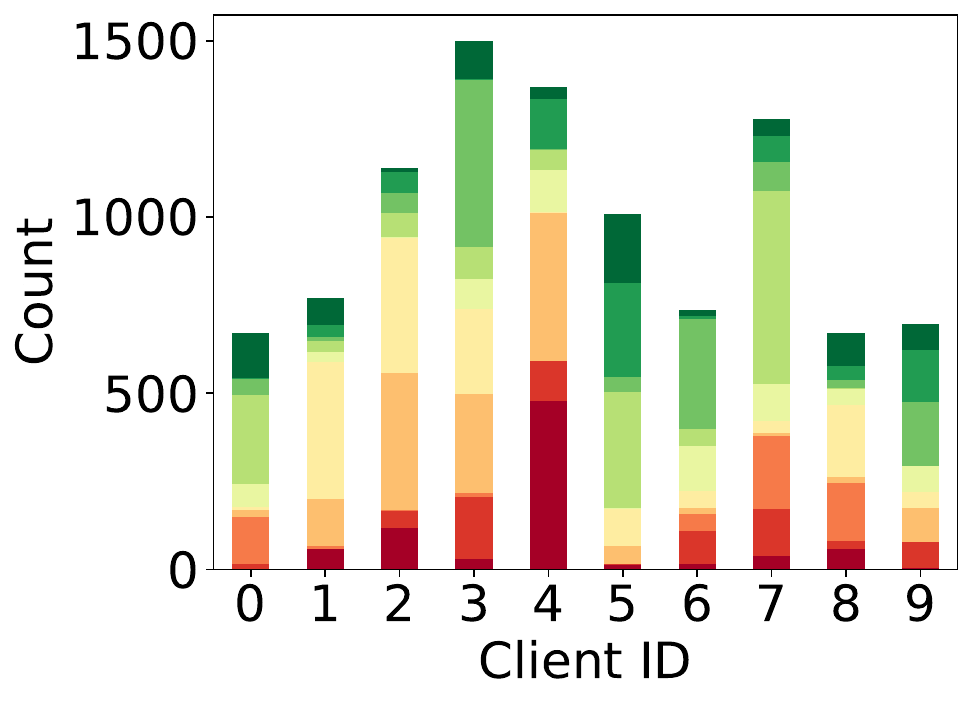}
  \label{fig:visual_cifar10}
}

\par 

\subfloat{
  \includegraphics[width=0.8\textwidth]{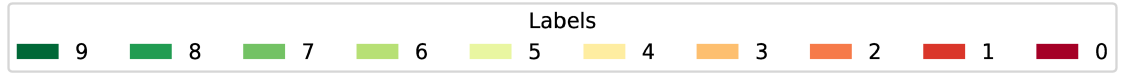}
}

\caption{Visualization of label distribution and size of local data sets of 10 clients on the data sets MNIST and Cifar-10 for varying non-iid scenarios. The data sets are split among 50 clients using a Dirichlet partitioning approach with Dirichlet concentration parameter $\dirichletparam \in \{30.0, 5.0, 0.5\}$, only 10 clients are shown. The degree of non-iid $\dirichletparam$ increases from left to right, which means that the data heterogeneity increases as well.}
\label{fig:visualnoniidmnistcifar}
\end{figure*}

\begin{figure*}[t]
\centering

\subfloat[{\textbf{Adult data set}}]{
  \includegraphics[width=0.25\textwidth]{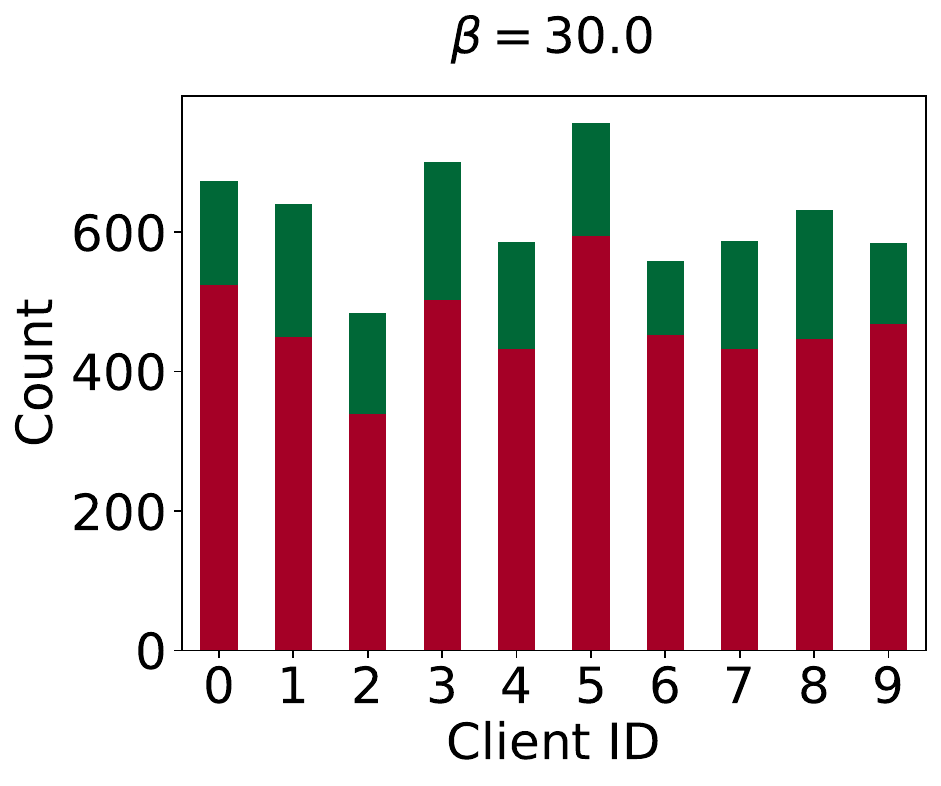}
  \includegraphics[width=0.25\textwidth]{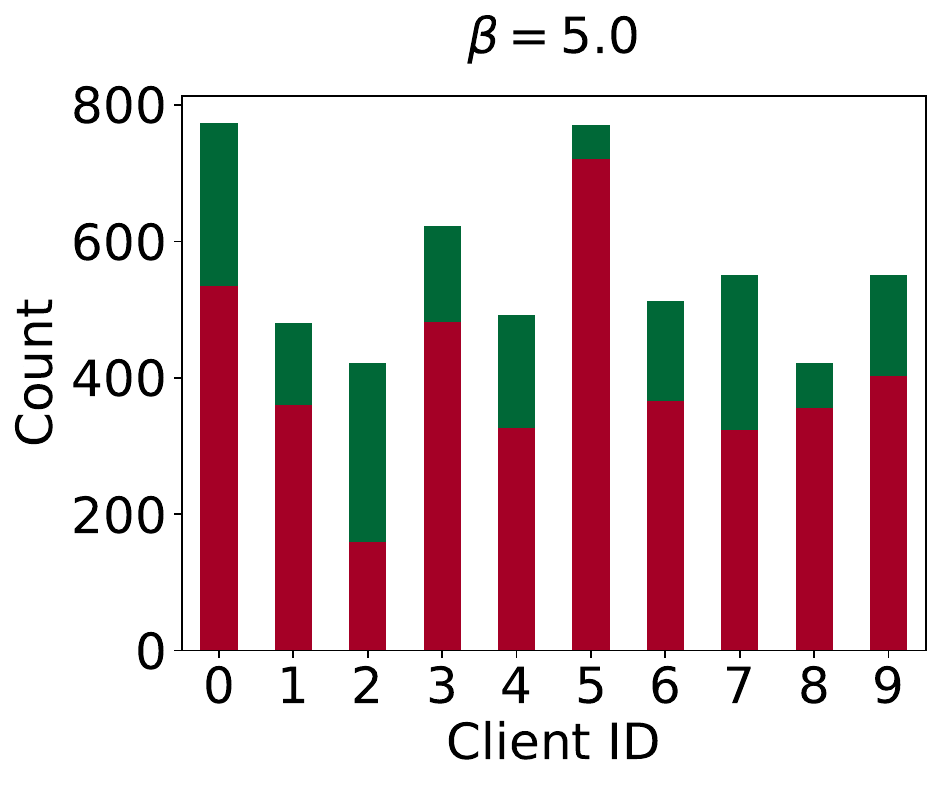}
  \includegraphics[width=0.25\textwidth]{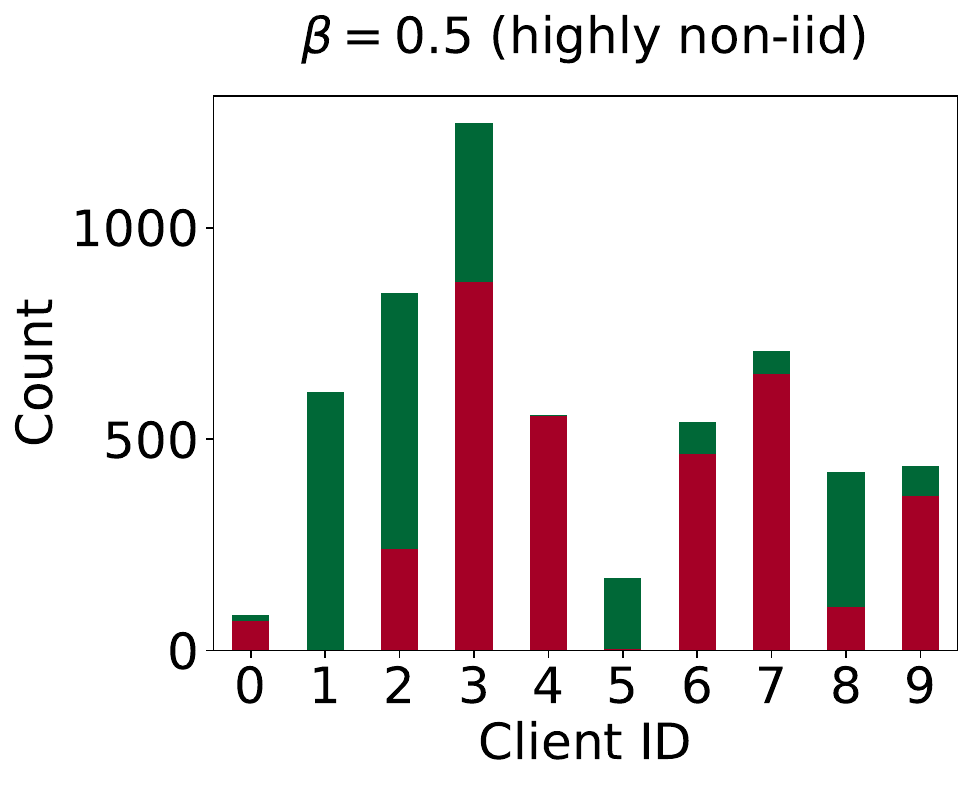}
  \label{fig:visual_adult}
}

\par

\subfloat{
  \includegraphics[width=0.2\textwidth]{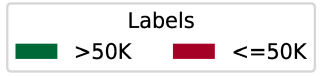}
}

\caption{Visualization of label distribution and size of local data sets of 10 clients on the Adult data set for varying non-iid scenarios. The data set is split among 50 clients using a Dirichlet partitioning approach with Dirichlet concentration parameter $\dirichletparam \in \{30.0, 5.0, 0.5\}$, only 10 clients are shown. The degree of non-iid $\dirichletparam$ increases from left to right, which means that the data heterogeneity increases as well.}
\label{fig:visualnoniidadult}
\end{figure*}

\section{Non-IID Accuracy Distributions}
\label{app:noniid_lossdists}
To provide further insights into the underlying structure of the hyperparameter search, we present the accuracy distributions for $\dirichletparam \in \{30.0, 5.0, 0.5\}$ and $\numclients \in \{50, 100, 250\}$ in \Cref{fig:lossdist_noniid}. For each hyperparameter, a federated model was trained to obtain $100$ accuracy values. 
Interestingly, the accuracy distributions do not differ much from the global evaluation reported in \Cref{fig:lossdist}. There are mainly two side effects that can be observed. First, for an increasing $\dirichletparam$ (top to bottom) the distributions shift slightly to the left. In other words, the maximum possible accuracy decreases for harder non-iid scenarios, which can also be observed in the reported results. Second, for a growing number of clients (right to left) the accuracy again decreases but this time only a bit, which is also visible in the results. Both side effects intuitively make sense, because both make the learning task harder, either by amplifying the non-iid degree or by reducing the number of data points per client.

\begin{figure*}[t]
\centering

\subfloat[{$\mathbf{\dirichletparam = 30.0}$}]{
  \includegraphics[width=\textwidth]{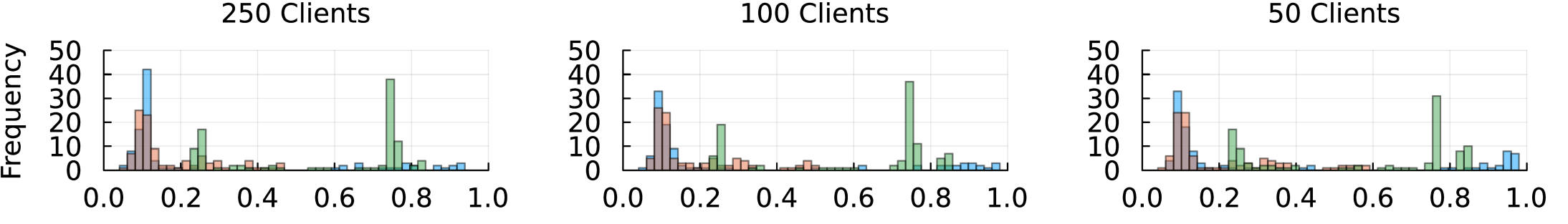}
  \label{fig:lossdist_a30.0}
}

\par

\subfloat[{$\mathbf{\dirichletparam = 5.0}$}]{
  \includegraphics[width=\textwidth]{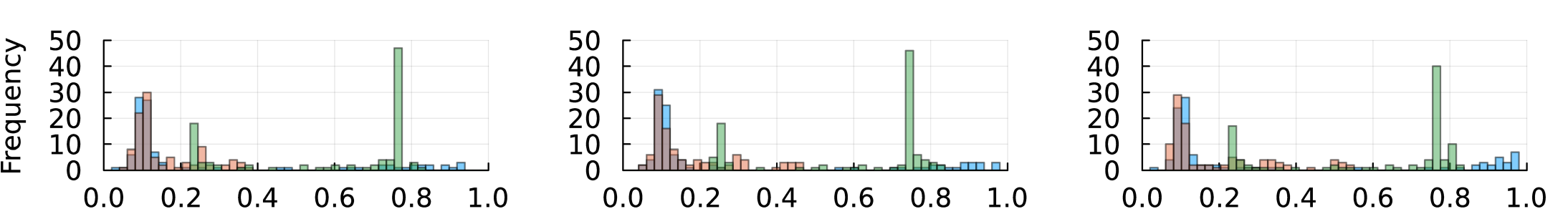}
  \label{fig:lossdist_a5.0}
}

\par

\subfloat[{$\mathbf{\dirichletparam = 0.5}$ \textbf{(Highly non-iid)}}]{
  \includegraphics[width=\textwidth]{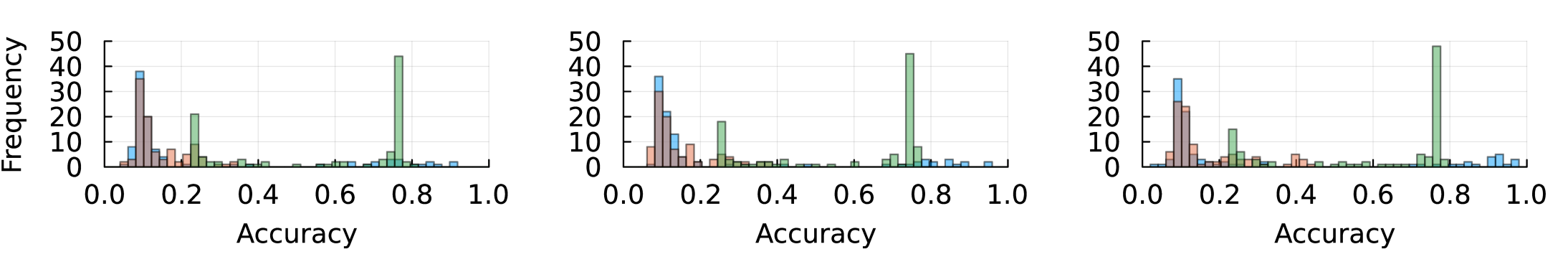}
  \label{fig:lossdist_a0.5}
}

\par

\subfloat{
  \includegraphics[width=\textwidth]{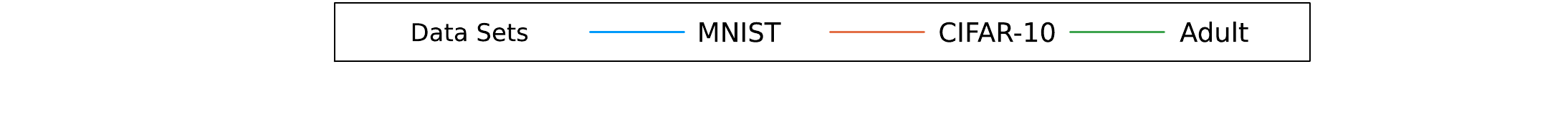}
}

\caption{A histogram over individual accuracy values when using federated learning with $\numclients \in \{50, 100, 250\}$. We train a model on each data set using SGD in different non-iid settings with Dirichlet concentration parameter $\dirichletparam \in \{30.0, 5.0, 0.5\}$ and $\dirichletparam=2.0$ for $\numclients=250$ on Adult
on the set of hyperparameters shown in \Cref{tbl::hps}.}
\label{fig:lossdist_noniid}
\end{figure*}

\section{Additional Non-IID Results}
\label{app:evalnoniid}
In this section, we present the results for the non-iid scenario that were moved here due to space limitations. The experimental setup is the same as in \Cref{sec:eval_privutiltradeoff} with $\numclients=250$ in \Cref{fig:evalnoniidN250} and $\numclients = 50$ in \Cref{fig:evalnoniidN50}. As a small side note, $\dirichletparam=2$ on Adult for $\numclients=250$ in the scenario of $\dirichletparam=0.5$, because this hard non-iid setting in combination with binary classification messes up the label and data distribution so much that no reasonable results could be obtained for Adult. 

The additional results uncover some interesting effects in the non-iid scenario. The performance of \opt{} decreases further for $\numclients=250$, due to smaller individual data sets of clients, and the opposite for $\numclients=50$. As expected, \tool{} benefits from a larger number of clients as this naturally increases the number of votes per hyperparameter, which makes it harder for the introduced noise to let a bad one win the voting.

\section{Raw Losses Instead of Binary Votes}
\label{app:rawlosses}
In our current design, each client participates in the binary vote by assigning a value of one to the $\numlocaltops$ best-performing hyperparameter candidates and zero to all remaining candidates. A natural alternative is to replace these binary indicators with real-valued scores derived from local performance, enabling a more fine-grained vote in which the globally best hyperparameter can, in principle, attain the true maximum. Using raw losses directly is inconvenient because lower loss corresponds to higher utility. Moreover, loss functions used in \tool{} must be bounded to avoid unbounded sensitivity. One simple normalization is to subtract each local loss from the maximum possible loss so that higher values correspond to better utility. In our setting, we can instead use accuracy as the score, since it is naturally bounded in the interval $[0,1]$ and therefore does not alter the sensitivity bound of the voting vector. We evaluate this strategy for $\numclients = 100$ and $\dirichletparam \in \{0.5, 5.0\}$ on the Cifar-$10$ data set, using the same experimental setup as in~\Cref{sec:eval_privutiltradeoff}.

The results are shown in~\Cref{fig:rawlosses}, where \emph{Binary} denotes the original \tool{} and \emph{Raw Loss} denotes the proposed real-valued variant. Overall, the accuracy-based voting still selects reasonable hyperparameters, as its resulting accuracy remains relatively close to the binary baseline across settings. However, in both non-iid scenarios, a clear degradation is visible for $\priveps \leq 3$. We attribute this to the effective value range of accuracy: while it is bounded between zero and one, realistic accuracies rarely approach one, so the per-client signal injected into the vote is substantially smaller than the worst-case change assumed by the sensitivity analysis (which permits $\numlocaltops$ entries to shift by one). Consequently, the signal-to-noise ratio deteriorates more quickly under privacy noise than in the binary scheme, amplifying the impact of noise and worsening the privacy-utility trade-off. Although a real-valued vote could still be beneficial in edge cases due to its finer granularity, our experiments indicate that it typically increases noise sensitivity and can significantly reduce utility at stricter privacy levels.

\begin{figure}[t]
    \centering
    \includegraphics[width=\linewidth]{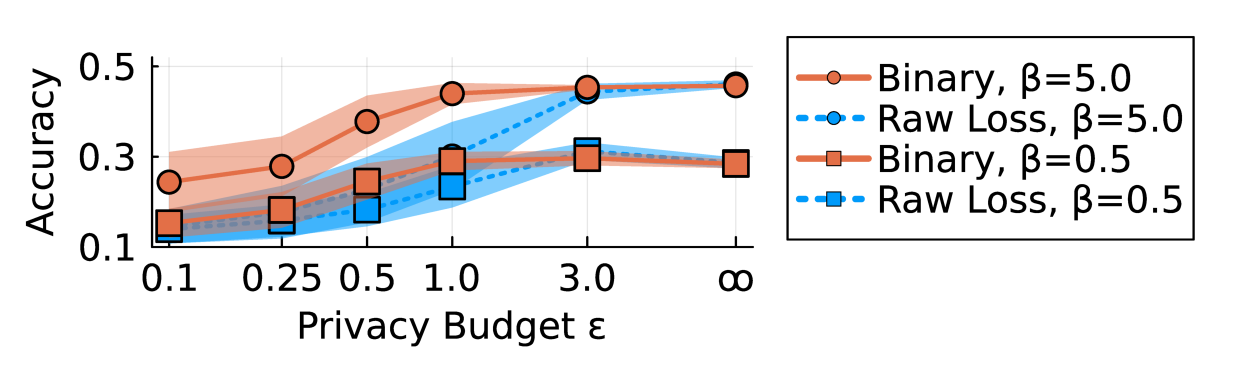}
    \caption{Comparison between default binary voting of \tool{} and an alternative approach where the accuracy for $\numlocaltops$ hyperparameters is reported directly instead of ones. The experiments are performed in Cifar-$10$ with $\numclients=100$. The utility of the alternative clearly degrades much faster for decreasing privacy budgets, which can be mainly attributed to a reduced signal-to-noise ratio as accuracy seldom reaches a value close to one.}
    \label{fig:rawlosses}
\end{figure}

\section{Comparison with Feathers}
\label{app:featherscomp}
We compare \tool{} with the privacy-preserving federated hyperparameter tuning that comes closest to our work, Feathers \cite{seng2022feathers}.
Feathers keeps a distribution over the hyperparameter candidates and updates this distribution to favor the configurations that attain the greatest decrease in clients' losses. Hyperparameter optimization is alternated with training phases that use the configuration that is rated the highest in the distribution.

However, a closer inspection of Feathers reveals that it does not preserve differential privacy; therefore we correct their privacy accounting before our evaluation. First, we clip the client loss values at $1.0$ to ensure bounded sensitivity of the losses which is necessary for guaranteeing DP. To calibrate the additive noise for the clipped loss values, we employ a Rényi DP accountant. We fix the number of DP approximated loss releases to $100$ to have an a priori bound on the privacy loss, which is necessary for satisfying formal DP guarantees. The vanilla version of Feathers uses the entropy of the hyperparameter distribution to select how many DP approximated losses are released, which only yields an a posteriori bound on the privacy loss, after a run completes.

In our evaluation, we compare \tool{} against Feathers on MNIST and Cifar-10 in the iid scenario with 250 clients. For Feathers, we configure all additional parameters as proposed in the paper \cite{seng2022feathers} and average the accuracy for 10 runs. The results are summarized in \Cref{fig:featherscomp}. For MNIST and $\varepsilon \geq 1.0$, Feathers can still attain its high accuracy. However, for smaller privacy budgets Feathers' accuracy decreases ($44.1\%$ accuracy for $\varepsilon=0.5$), while \tool{} still attains the same high accuracy. This is as expected because Feathers uses a local DP approach which adds unnecessarily large amounts of noise, quickly yielding an almost random draw from the hyperparameter candidates. This is also apparent for Cifar-10, where the accuracy of Feathers begins to decrease at $\varepsilon=3.0$ ($38.9\%$). This matches our expectation: \Cref{fig:lossdist} shows that for Cifar-10, fewer hyperparameter candidates perform well. Thus, Feathers' suboptimal signal-to-noise ratio likely leads it to select configurations that perform suboptimally.

\section{Parallelization Baseline}
\label{app:parabaseline}
\begin{figure}[t]
    \centering
    \includegraphics[width=\linewidth]{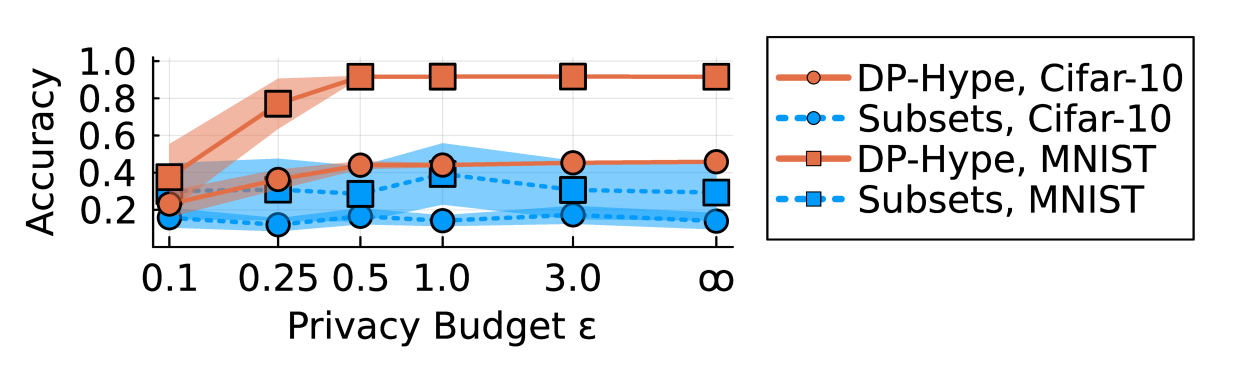}
    \caption{Comparison between \tool{} and a simple parallelization strategy where clients are randomly assigned into subsets each evaluating a single hyperparameter. Per subset noisy losses are aggregated and the server selects the hyperparameter with optimal loss. Experiments are performed on MNIST and Cifar-$10$ for $n=250,\dirichletparam=30$ and $n=100$, iid, respectively. \tool{} clearly outperforms the alternative strategy. The main reason for this is that the per-hyperparameter signal is just too small to overcome even small noise, and that even for no privacy the variance per hyperparameter evaluation is too large to let the best candidates win reliably.}
    \label{fig:subsets}
\end{figure}
\tool{} currently requires every client to train and evaluate a model for each hyperparameter candidate locally, which can be computationally expensive. To reduce this resource overhead, one can divide clients into random subsets and assign a single hyperparameter to each subset. The server then aggregates the reports by averaging the noisy losses per subset and selects the hyperparameter with the best average performance. This reduces the per-client workload to training just one model, but it also substantially lowers the signal-to-noise ratio for each hyperparameter, since each candidate is evaluated by only a fraction of the clients. From a privacy perspective, this strategy no longer involves a parameter $\numlocaltops$, and the sensitivity of the overall voting vector becomes $\sqrt{2}$ because replacing the data of a single client can affect at most two entries by $1$.

Results for \tool{} and this alternative, which we denote by \emph{Subsets}, are shown in~\Cref{fig:subsets}. Across both scenarios, \tool{} clearly outperforms \emph{Subsets}. Even in the MNIST iid setting, which we included to make the task as favorable as possible for \emph{Subsets}, the method does not exceed performance comparable to guessing. We attribute this to two factors. First, the poor performance even in the non-private regime indicates that evaluating each hyperparameter on only a small number of clients (e.g., two- or three-fold for $\numclients = 250$) does not yield a reliable signal: even strong hyperparameters tend to perform well for a majority of clients rather than uniformly for every client, and this majority effect is suppressed when evaluations are highly fragmented. Second, the signal-to-noise ratio is inherently smaller than in \tool{}, and this disadvantage compounds under privacy noise. Indeed, as observed in~\Cref{app:rawlosses}, using plain accuracy already provides a weak signal, and reducing the number of clients per hyperparameter further amplifies the impact of noise.

\begin{figure}[t]
    \centering
    \includegraphics[width=\linewidth]{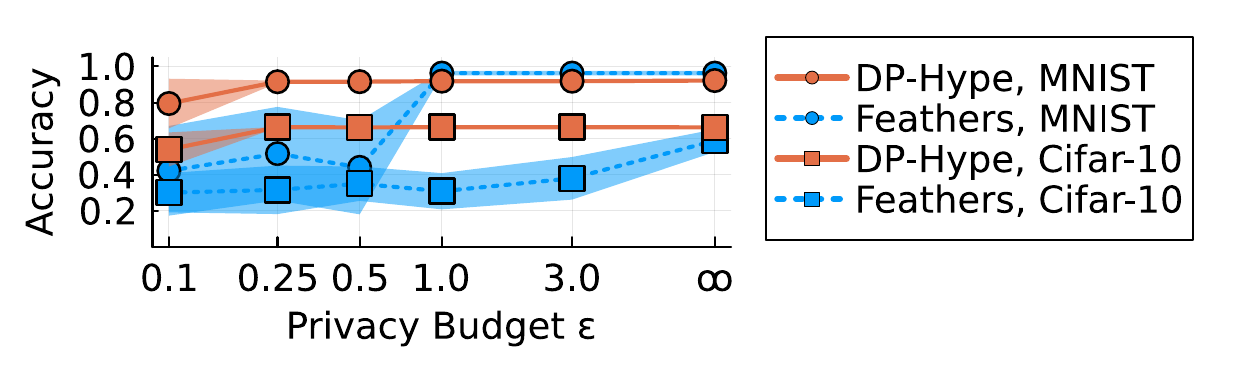}
    \caption{Empirical results in terms of Accuracy for evaluating \tool{} and Feathers \cite{seng2022feathers} on the data sets MNIST and Cifar-10 with $\numclients = 250$ clients and different privacy budgets. \tool{} clearly attains a superior privacy-utility trade-off compared to Feathers. Feathers often selects hyperparameters almost randomly, yielding suboptimal accuracy.}
    \label{fig:featherscomp}
\end{figure}

\begin{figure*}
\centering

\subfloat[{\textbf{MNIST data set}}]{
  \includegraphics[width=0.95\textwidth]{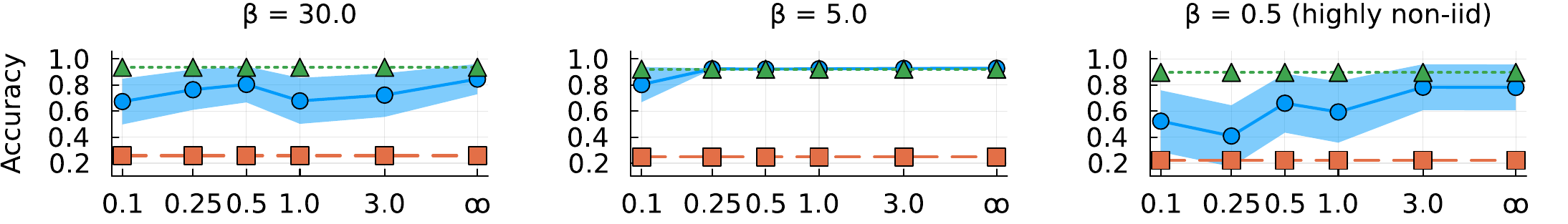}
  \label{fig:evalnoniidmnistN250}
}

\par

\subfloat[{\textbf{Cifar-$\mathbf{10}$ data set}}]{
  \includegraphics[width=0.95\textwidth]{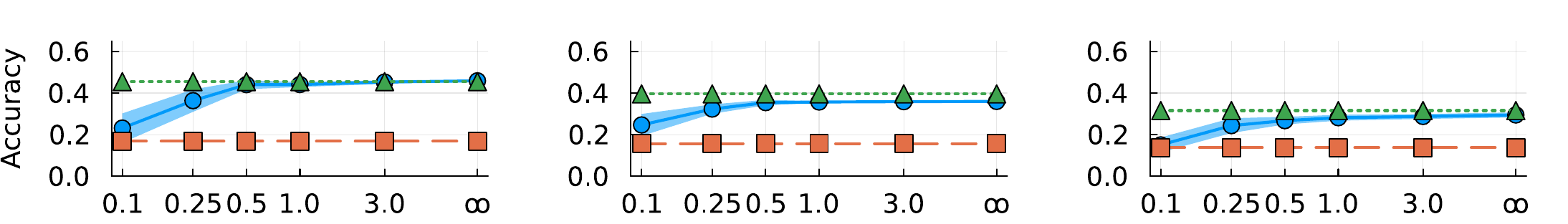}
  \label{fig:evalnoniidcifar10N250}
}

\par

\subfloat[{\textbf{Adult data set}}]{
  \includegraphics[width=0.95\textwidth]{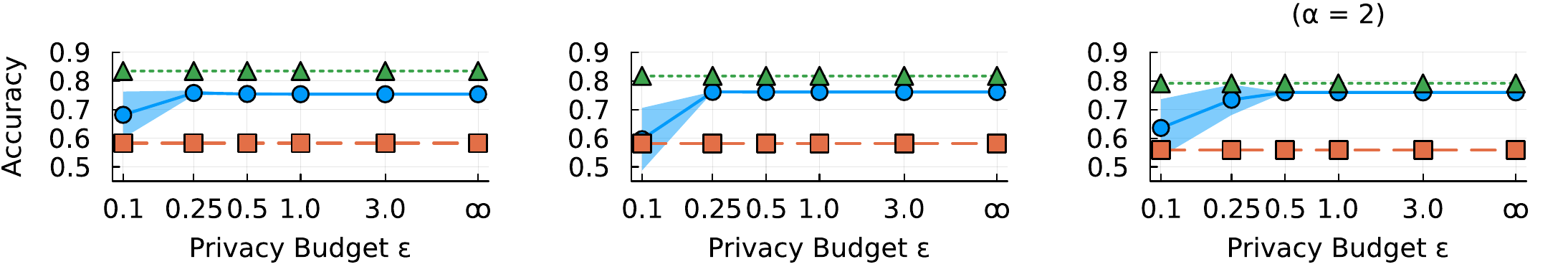}
  \label{fig:evalnoniidadultN250}
}

\par

\subfloat{
  \includegraphics[width=0.95\textwidth]{pets26/imgs/eval_iid_Legend_allN.pdf}
}

\caption{Empirical results in terms of Accuracy for evaluating \tool{}, \randguess{} and \opt{} on the data sets MNIST, Cifar-10 and Adult for $\numclients=250$ and different privacy budgets. We train a model in different non-iid scenario with Dirichlet concentration parameter $\dirichletparam \in \{30.0, 5.0, 0.5\}$ and $\dirichletparam = 2.0$ for $\numclients = 250$ on Adult. The degree of non-iid $\dirichletparam$ increases from left to right, which means that the data heterogeneity increases as well.}
\label{fig:evalnoniidN250}
\end{figure*}

\begin{figure*}
\centering

\subfloat[{\textbf{MNIST data set}}]{
  \includegraphics[width=0.95\textwidth]{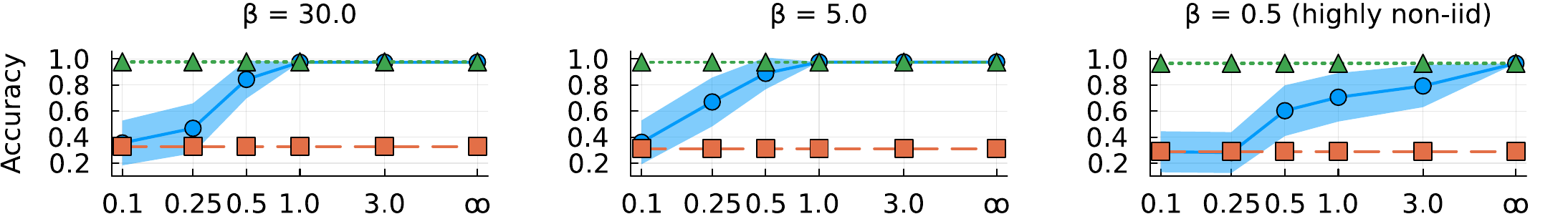}
  \label{fig:evalnoniidmnistN50}
}

\par

\subfloat[{\textbf{Cifar-$\mathbf{10}$ data set}}]{
  \includegraphics[width=0.95\textwidth]{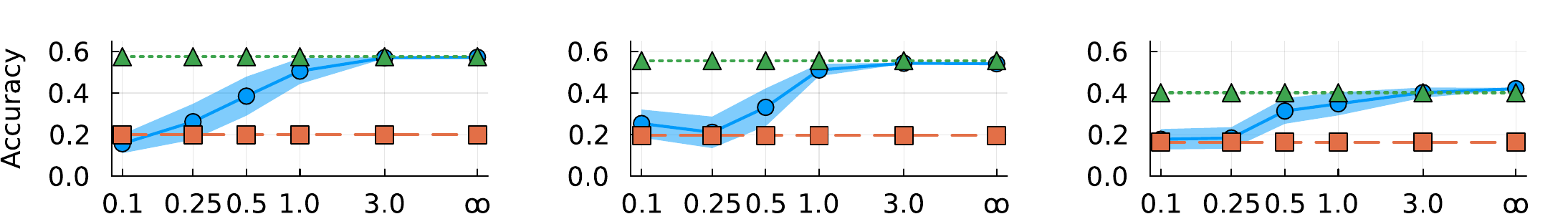}
  \label{fig:evalnoniidcifar10N50}
}

\par

\subfloat[{\textbf{Adult data set}}]{
  \includegraphics[width=0.95\textwidth]{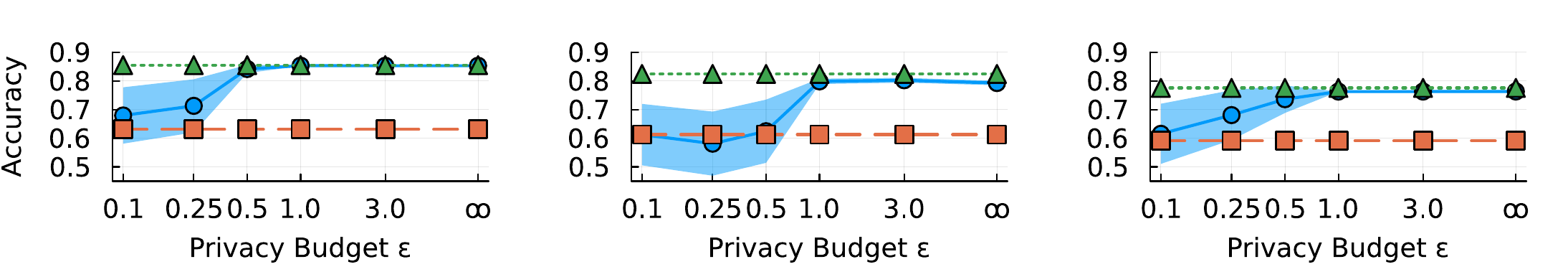}
  \label{fig:evalnoniidadultN50}
}

\par

\subfloat{
  \includegraphics[width=0.95\textwidth]{pets26/imgs/eval_iid_Legend_allN.pdf}
}

\caption{Empirical results in terms of Accuracy for evaluating \tool{}, \randguess{} and \opt{} on the data sets MNIST, Cifar-10 and Adult for $\numclients=50$ and different privacy budgets. We train a model in different non-iid scenario with Dirichlet concentration parameter $\dirichletparam \in \{30.0, 5.0, 0.5\}$. The degree of non-iid $\dirichletparam$ increases from left to right, which means that the data heterogeneity increases as well.}
\label{fig:evalnoniidN50}
\end{figure*}

\end{appendices}

\end{document}